\documentclass[preprint,12pt]{elsarticle}

\usepackage{amssymb}
\usepackage{amsmath}

\usepackage{lineno}

\usepackage{color}
\usepackage{array}
\usepackage{bm}
\usepackage{algorithmic}
\usepackage{algorithm}
\usepackage{booktabs}
\usepackage{subcaption}

\newtheorem{corollary}{Corollary}
\newenvironment{proof}{{\it Proof.}}{\hfill $\blacksquare$\par}

\journal{Neurocomputing}

\begin{document}
\begin{sloppypar}

\begin{frontmatter}

\title{Zero-Shot Neural Architecture Search with Weighted Response Correlation}

\author[label1]{Kun Jing\corref{cor1}}
\ead{jingkun@ahu.edu.cn}
\author[label2]{Luoyu Chen}
\author[label3]{Jungang Xu\corref{cor1}}
\ead{xujg@ucas.ac.cn}
\author[label1]{Jianwei Tai}
\author[label4]{Yiyu Wang}
\author[label5]{Shuaimin Li}
\cortext[cor1]{Corresponding authors.}
\affiliation[label1]{organization={School of Internet, Anhui University},
            city={Hefei},
            country={China}}
\affiliation[label2]{organization={Alibaba Group Holding Limited},
            city={Hangzhou},
            country={China}}
\affiliation[label3]{organization={School of Computer Science and Technology, University of Chinese Academy of Sciences},
            city={Beijing},
            country={China}}
\affiliation[label4]{organization={Alibaba International Digital Commerce},
            city={Hangzhou},
            country={China}}
\affiliation[label5]{organization={Shenzhen Key Laboratory for High Performance Data Mining, Shenzhen Institutes of Advanced Technology, Chinese Academy of Sciences},
            city={Shenzhen},
            country={China}}

\begin{abstract}
Neural architecture search (NAS) is a promising approach for automatically designing neural network architectures. However, the architecture estimation of NAS is computationally expensive and time-consuming because of training multiple architectures from scratch. Although existing zero-shot NAS methods use training-free proxies to accelerate the architecture estimation, their effectiveness, stability, and generality are still lacking. We present a novel training-free estimation proxy called weighted response correlation (WRCor). WRCor utilizes correlation coefficient matrices of responses across different input samples to calculate the proxy scores of estimated architectures, which can measure their expressivity and generalizability. Experimental results on proxy evaluation demonstrate that WRCor and its voting proxies are more efficient estimation strategies than existing proxies. We also apply them with different search strategies in architecture search. Experimental results on architecture search show that our zero-shot NAS algorithm outperforms most existing NAS algorithms in different search spaces. Our NAS algorithm can discover an architecture with a 22.1\% test error on the ImageNet-1k dataset within 4 GPU hours. All codes are publicly available at https://github.com/kunjing96/ZSNAS-WRCor.git.
\end{abstract}


\end{frontmatter}


\section{Introduction}

The success of deep learning in various fields \cite{NASSurvey}, especially computer vision, causes a surge in demand for designing neural architectures. Designing neural architectures manually requires extensive expertise and time investment. Neural architecture search (NAS) \cite{NASSurvey, EvoAAE, MODEOCNN, AER, MR-DARTS, SlimDARTS, DESEvo} offers a potential solution for automatically designing neural architectures across various domains, eliminating the need for human involvement.

The architecture estimation strategy is one of the core components of NAS. Early NAS researchers \cite{NAS, NASNet, AmoebaNet} employ standard training and approximate training proxy methods for architecture estimation called multi-shot NAS. However, these methods usually require substantial computational resources and time to train numerous architectures from scratch.
For efficiency, few-shot NAS methods \cite{PNAS, GMAENAS, GATES} utilize architecture performance predictors implemented by neural networks to predict the accuracies or relative performance values of neural architectures. For reliable predictors, these methods have to train a sufficient number of neural architectures.
Besides, one-shot NAS methods \cite{ENAS, DARTS, MR-DARTS, SlimDARTS} propose a parameter-sharing technology, which maintains a set of network parameters of an over-parameterized supernet during the search process and estimates sub-networks using the supernet parameters. The training of one-shot supernets is challenging \cite{AER} due to the intricate coupling of subnets, compromising quality and robustness.
In summary, the training is still inevitable for architecture estimation in few-shot and one-shot NAS.

Recently, zero-shot NAS methods \cite{NASWOT, EPENAS, ZeroShotNAS, ZeroCost, ZenNAS, ZiCo} propose utilizing several matrix calculations or model inferences to estimate neural architectures without training within seconds. They are known as training-free estimation proxies.
Although zero-shot NAS methods are more efficient, existing training-free estimation proxies fail to consistently and robustly outperform two naive proxies \cite{ZeroShotNAS, DeeperZeroCost, ZiCo}, i.e., the number of parameters (Params) and floating point operations (FLOPs). Some proxies have specific requirements \cite{NASWOT, TENAS, ZenNAS} for neural architectures. These undermine the effectiveness, stability, and generality of zero-shot NAS methods.

According to zero-shot NAS methods \cite{NASWOT, EPENAS, ZeroShotNAS, ZeroCost, ZenNAS, ZiCo} and potential proxies \cite{SNIP, GraSP, SynFlow, Fisher, LyrDynIsometry} from other fields, the performances of neural architectures mainly depend on their expressivity and generalizability. 
Based on this conclusion, we propose a novel training-free estimation proxy called weighted response correlation (WRCor), as shwon in Figure~\ref{fig: OVERVIER}, which can compute the response (activation and gradient) correlation coefficient matrices of different inputs throughout different layers of the network and aggregate them with layer-wise weights to calculate the proxy score. We clarify that a higher proxy score means the responses of different inputs are more linearly independent in a given network, which is positively related to its better performance. For further improvement, we propose two voting proxies, SPW and SJW, using the majority voting rule.
The experimental results on proxy evaluation across three datasets and two NAS benchmarks verify the superiority of our training-free estimation proxies over existing proxies. The experimental results on architecture search across three search spaces demonstrate that our zero-shot NAS algorithms with various search strategies and our training-free estimation proxies outperform existing NAS algorithms for image recognition.
\begin{figure*}[t]
  \centering
  \includegraphics[width=\linewidth]{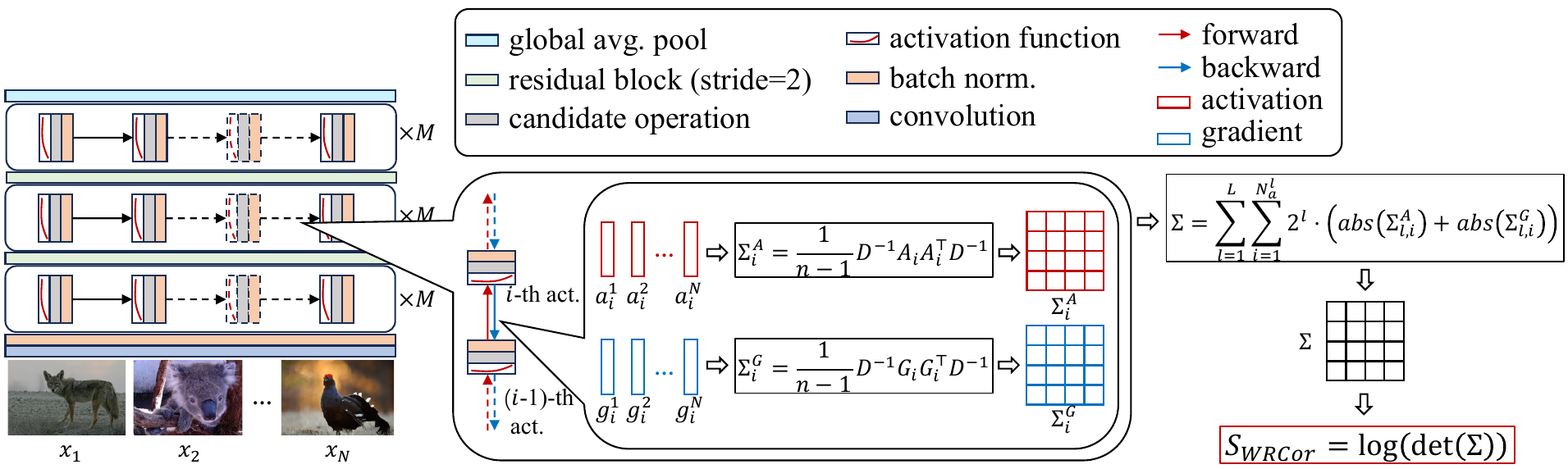}
  \caption{The overview of our proposed training-free proxy.}
  \label{fig: OVERVIER}
\end{figure*}

Our contributions are summarized as follows.
\begin{itemize}
\item We propose an efficient, robust, and generic training-free estimation proxy (WRCor) that offers a unified form for measuring both layer-wise expressivity and generalizability of neural architecture.
\item We propose two majority voting proxies (SPW and SJW) to improve the performance and stability of architecture estimation.
\item We use our training-free estimation proxies to enhance various search strategies for more efficient NAS.
\item Our zero-shot NAS algorithm RE-SJW achieves a highly competitive result of 22.1\% top-1 test error on the ImageNet-1k dataset within just four GPU hours.
\end{itemize}

\section{Related Work}

Training-free proxy is the key component of zero-shot NAS. The goal of training-free proxies is to estimate the performance scores of neural architectures without training.
NASWOT \cite{NASWOT} uses the number of linear regions of the input space to measure the expressivity of neural architectures.
However, NASWOT ignore the trainability of architectures, which imply how effective a network can be optimized via gradient descent. A network can achieve high performance only if the function it can represent is complex enough and at the same time, it can be effectively trained by gradient descent.
In addition to expressivity, TE-NAS \cite{TENAS} propose to analyze the spectrum of the neural tangent kernel for measuring the trainability of architectures.
According to this assumption that an untrained network with low correlation between different data points, where data points of the same category are closer to each other, can easily learn to distinguish the two data points during training,
EPE-NAS \cite{EPENAS} also proposes an improvement of Jacobian correlation (JacCor) \cite{NASWOT} called EPE, which involves both intra- and inter-class correlation.
Beside, many saliency metrics of pruning-at-initialization are extended for zero-shot NAS \cite{ZeroCost}, such as SNIP \cite{SNIP}, GraSP \cite{GraSP}, SynFlow \cite{SynFlow}, and Fisher \cite{Fisher, BlockSwap}. These proxies rely on the gradient w.r.t the parameters of neural networks, which are different approximations of the Taylor expansions of neural networks.
Zen-score (Zen) \cite{ZenNAS} approximates $\Phi$-score (Phi) against model-collapse in very deep networks and can directly assert the model complexity.
Although these proxies perform well, their approximate expressions limit their effectiveness.
None of these zero-shot proxies can actually work consistently better than a naive proxy, i.e., Params and FLOPs. Based on theoretical analysis,
ZiCo \cite{ZiCo} reveals that the network with a high training convergence speed and generalization capacity should have high absolute mean values and low standard deviation values for the gradient w.r.t the parameters across different samples.

There are also some potential proxies from other fields for training-free architecture estimation.
Euclidean norm of parameters (PNorm) and Euclidean norm of parameter gradients (PGNorm) are the simplest saliency metrics for parameter pruning. We use the sum of Euclidean norms of all parameters or parameter gradients for architecture estimation.
Barlow Twins loss (BT) \cite{BarlowTwins} can measure the cross-correlation matrix between the outputs of two identical networks fed with distorted versions of a sample. We use Barlow Twins loss multiplied by -1 as a training-free proxy.
Bundle entropy (BndlEnt) \cite{BundleEntropy} is proposed to measure the sample conflict in each layer. We use the average value of bundle entropies for all network layers as a training-free proxy.
Layer-wise dynamical isometry (LyrDynIso) \cite{LyrDynIsometry} is defined as the singular values of the Jacobian matrix $\bm{J}^{l-1,l}\in\mathbb{R}^{d_l \times d_{l-1}}$ of layer $l$ are concentrated near 1 for all layers, which guarantees faithful signal propagation in the network.
We also use a more efficient proxy DynIso that takes the Jacobian matrix $\bm{J}^{0, L}$ of the network into account instead of the layer-wise Jacobian matrix.
The Jacobian matrix can measure the change in network output when the input changes. We use $l_1$-norm of the Jacobian matrix as a training-free proxy NetSens.
Although these potential proxies performs well in some scenes, their effectiveness in all situations are not experimentally verified or theoretically analyzed.

These above training-free proxies can significantly accelerate the NAS process because they can estimate the performances of architectures without training. Compared with conventional NAS methods, zero-shot NAS is more efficient NAS method. However, there are three disadvantages for these above training-free proxies: 1) they fail to consistently and robustly outperform two naive proxies \cite{ZeroShotNAS, DeeperZeroCost, ZiCo}, i.e., Params and FLOPs; 2) they cannot consistently and robustly outperform other existing proxies in all situations; 3) some proxies have specific requirements for neural architectures, e.g., NASWOT \cite{NASWOT} and TENAS \cite{TENAS} can only estimate the performance of CNN architecture and ZenNAS \cite{ZenNAS} requires a CNN architecture with ReLU as the activation function.
These undermine the effectiveness, stability, and generality of existing zero-shot NAS methods with training-free proxies. To solve these limitations, we analyze the behavior of neural networks and propose an efficient, robust, and generic training-free estimation proxy called WRCor that offers a unified form for measuring both layer-wise expressivity and generalizability.

\section{Methodology}

For clarity and convenience, we define \emph{activations} as the hidden feature maps and \emph{gradients} as the gradients w.r.t the hidden feature maps. These two are collectively referred to as \emph{responses}.

\subsection{Expressivity}

The \emph{expressivity} refers to the ability of neural architecture to distinguish different inputs. Since neural networks are adept at handling problems that are even more linearly separable, different samples with linearly independent activations can be distinguished more easily, i.e., a neural architecture has greater expressivity. The correlation coefficient, which measures the degree of linear correlation, serves as a suitable metric for the correlation of activations. A lower correlation coefficient of activations across different samples indicates their lower linear correlation.
Inspired by this, we propose utilizing the correlation coefficient of activations (correlation of activations, ACor) across different inputs for evaluating the expressivity of neural architectures.
Specifically, given a neural network with $N_a$ non-linear activation units and a randomly sampled minibatch $\bm{X}=(\bm{x}_1, \bm{x}_2, \cdots, \bm{x}_N)^\top$ with batch size $N$, we collect the $i$-th flattened activation $\bm{a}^n_i\in\mathbb{R}^{d_i}$ of the $n$-th sample by forward inference, where $d_i$ represents the dimension of the flattened activations. We compute the correlation coefficient matrices $\bm{\Sigma}^A_i\in\mathbb{R}^{N \times N}$ for activations $\bm{A}_i=(\bm{a}_i^1, \bm{a}_i^2, \cdots, \bm{a}_i^N)^\top$.
The main-diagonal elements of the correlation coefficient matrix $\bm{\Sigma}^A_i$ are always 1. Based on the above conclusion, networks with better expressivity have a lower correlation of the activations across different inputs, i.e., the smaller non-main-diagonal elements in the correlation coefficient matrix $\bm{\Sigma}^A_i$.

\subsection{Generalizability}

The \emph{generalizability} refers to the ability to perform well on unseen data after being trained via gradient descent,
which is equally important as expressivity.
Intuitively, it is believed that gradients that exhibit greater linear dependence facilitate convergence \cite{TFGB} because when gradients are more linearly dependent, they point in similar directions, potentially leading to a more straightforward descent towards a minimum in the loss landscape. However, if the gradients are consistently aligned across different inputs, this could potentially lead to issues such as poor generalization or convergence to suboptimal minima. In essence, independent gradient directions might encourage to exploration of a broader range of solutions during training, thus enhancing its generalizability.
Similarly, we also propose utilizing the correlation coefficient of gradients (correlation of gradients, GCor) across different inputs for evaluating the generalizability of neural architectures.
By collecting the $i$-th flattened gradients $\bm{g}^n_i$ of the $n$-th sample by backward propagation, we compute the correlation coefficient matrices $\bm{\Sigma}^G_i$ for gradients $\bm{G}_i=(\bm{g}_i^1, \bm{g}_i^2, \cdots, \bm{g}_i^N)^\top$.
The main-diagonal elements of the correlation coefficient matrix $\bm{\Sigma}^G_i$ are always 1. Based on the above conclusion, networks with better generalizability have a lower correlation of the gradients across different inputs, i.e., smaller non-main-diagonal elements in the correlation coefficient matrix $\bm{\Sigma}^G_i$.

\subsection{Training-Free Proxies}

Our training-free proxies rely on the correlation of responses.
Given that informative responses tend to yield superior proxies, we consider three candidates: (a) post-activation responses; (b) pre-normalization responses; and (c) pre-activation or post-normalization responses.
Option (a) is inappropriate because the non-linear activation units of untrained networks are often saturated, resulting in information loss. For instance, pre-activation activations in ReLU networks contain a large number of negative values, which become zeros after passing through ReLU activation units, leading to the loss of information related to negative values. Option (b) is biased due to the lack of normalization. Therefore, we propose using option (c), which can provide unbiased and the most informative responses. We compute the correlation coefficient matrices $\bm{\Sigma}^A_i$ and $\bm{\Sigma}^G_i$ by collecting the flattened pre-activation activations and gradients of all samples.

\subsubsection{Correlation of Responses}

For architecture estimation, we map these correlation coefficient matrices of responses (including activations and gradients) to a scalar score.

\begin{corollary}
  For randomly sampled inputs, the correlation coefficients between their responses through a randomly initialized neural network are almost equal. \label{corollary1}
\end{corollary}
\begin{proof}
  Due to the i.i.d. parameters of a randomly initialized neural network and different inputs from the same distribution, the distributions of activations in each layer are close. Besides, the distributions of gradients in each layer are also close without considering different scale factors from the gradients w.r.t the final network outputs. Given a randomly initialized neural network and sufficient inputs, the correlation of responses depends just on neural architectures. Therefore, the correlation coefficients between responses are almost equal.
\end{proof}

According to the Corollary~\ref{corollary1}, we aggregate all correlation coefficient matrices of pre-activation activations or pre-activation gradients by
\begin{equation}
  \bm{\Sigma}^{A/G}=\sum_{i=1}^{N_a}\mathrm{abs}(\bm{\Sigma}^{A/G}_i) \approx
  \begin{bmatrix}
    N_a & N_ax & \cdots & N_ax \\
    N_ax & N_a & \cdots & N_ax \\
    \vdots & \vdots & \ddots & \vdots \\
    N_ax & N_ax & \cdots & N_a \\
  \end{bmatrix}
\end{equation}
where $x\in[0,1]$. The aggregated correlation coefficient matrix can be regarded as the non-normalized expectation of the correlation coefficient matrices.

\begin{corollary}
  For a $N$-order symmetric matrix ($N\geq2$) where all main-diagonal elements are 1 and all non-main-diagonal elements are $x\in[0,1]$, its determinant value equals $[1+(N-1)x](1-x)^{N-1}$. It monotonically decreases w.r.t $x$. It attains the maximum value of 1 iff $x=0$ and the minimum value of 0 iff $x=1$.\label{corollary2}
\end{corollary}
\begin{proof}
We compute the determinant value of the matrix by
\begin{eqnarray}
  \det\left(\begin{bmatrix}
    1 & x & \cdots & x \\
    x & 1 & \cdots & x \\
    \vdots & \vdots & \ddots & \vdots \\
    x & x & \cdots & 1 \\
  \end{bmatrix}\right)
  =& [1+(N-1)x]
  \det\left(\begin{bmatrix}
    1 & 1 & \cdots & 1 \\
    x & 1 & \cdots & x \\
    \vdots & \vdots & \ddots & \vdots \\
    x & x & \cdots & 1 \\
  \end{bmatrix}\right)\\
  =& [1+(N-1)x]
  \det\left(\begin{bmatrix}
    1 & 1 & \cdots & 1 \\
    0 & 1-x & \cdots & 0 \\
    \vdots & \vdots & \ddots & \vdots \\
    0 & 0 & \cdots & 1-x \\
  \end{bmatrix}\right)\\
  =& [1+(N-1)x](1-x)^{N-1}\label{formula:DET}
\end{eqnarray}
The first-order derivative
\begin{equation}
  \frac{\partial [1+(N-1)x](1-x)^{N-1}}{\partial x} = -N(N-1)(1-x)^{N-2}x
\end{equation}
of Formula~\ref{formula:DET} remains non-positive on $x\in[0,1]$. Therefore, Formula~\ref{formula:DET} exhibits a strict decrease on $x\in[0,1]$ w.r.t $x$, attaining its maximum value of 1 at $x=0$ and its minimum value of 0 at $x=1$.
\end{proof}

According to the Corollary~\ref{corollary2}, we map the aggregated correlation coefficient matrix $\bm{\Sigma}^{A/G}$ to a scalar score by
\begin{eqnarray}
  S_{ACor/GCor} = \log(\det(\bm{\Sigma}^{A/G}))
  \approx& N_a^N
  \det\left(\begin{bmatrix}
    1 & x & \cdots & x \\
    x & 1 & \cdots & x \\
    \vdots & \vdots & \ddots & \vdots \\
    x & x & \cdots & 1 \\
  \end{bmatrix}\right)\\
  =& N_a^N [1+(N-1)x](1-x)^{N-1}
\end{eqnarray}
where the logarithm function can avoid the overflow caused by the continuous multiplication of large numbers.

After that, a higher score implies that the activations or gradients of different inputs have a lower linear correlation because the non-main-diagonal elements are close to 0. Therefore, a higher score corresponds to a better architecture.

To improve our proxy, we use the correlation of pre-activation responses as our training-free proxy by computing
\begin{eqnarray}
  S_{RCor} = \log(\det(\bm{\Sigma})),\\
  \bm{\Sigma}=\sum_{i=1}^{N_a}\mathrm{abs}(\bm{\Sigma}^A_i)+\mathrm{abs}(\bm{\Sigma}^G_i)
  \label{formula:ACTGRADCOR}
\end{eqnarray}
Similarly, according to the Corollary~\ref{corollary1} and \ref{corollary2}, a high RCor score also indicates a superior architecture.

Our training-free proxies are significantly more flexible and general than NASWOT \cite{NASWOT}, TENAS \cite{TENAS}, and ZenNAS \cite{ZenNAS} because they apply to any neural architecture. Besides, we compute the score using matrix aggregation rather than accumulating scores of layers, which can avoid a large number of calculations for matrix determinants.

\subsubsection{Weighted Correlation of Responses}\label{section:WEIGHTED}

Furthermore, our observations indicate that state-of-the-art networks often fail to distinguish different samples based on bottom-level features and rely on top-level features for clear discrimination. These networks usually capture common bottom-level features and independent top-level features. For superior neural architectures, the correlation of top-level responses is lower than that of bottom-level responses across different inputs. Based on our observation and hypothesis, we allow non-strict linear independence in bottom-level responses and focus on the correlation of top-level responses. Specifically, when aggregating the correlation coefficient matrices, we assign higher weights to the correlation coefficient matrices of top-level responses than those of bottom-level responses for weighting them. 
According to the above conclusion, we expect the importance of the responses to significantly increase with the number of layers. Therefore, we choose to exponentially weight these matrices.  Compared to linear weighting and quadratic weighting, it places more emphasis on top-level responses. Their importances or weights increase geometrically across successive layers.

Given an $L$-layer neural network, where the $l$-th layer has $N^l_a$ non-linear activation units, we summarize these matrices using element-wise weighted summation. We assign exponentially increasing weights from the bottom to the top, i.e.,
\begin{eqnarray}
  S_{WRCor} = \log(\det(\bm{\Sigma}))\\
  \bm{\Sigma}=\sum_{l=1}^L\sum_{i=1}^{N^l_a} 2^l \cdot (\mathrm{abs}(\bm{\Sigma}^A_{l,i})+\mathrm{abs}(\bm{\Sigma}^G_{l,i}))
  \label{formula:WEIGHTACTGRADCOR}
\end{eqnarray}
This ensures that the top layers have a greater influence in evaluating deeper networks, while all layers contribute to evaluating shallower ones, which aligns with our initial expectations. According to the Corollary~\ref{corollary1} and \ref{corollary2}, it can be proven that a higher WRCor score closely corresponds to a superior architecture.

In summary, compared to existing training-free proxies, our WRCor proxy has three differences and advantages: 1) ours can consider both the expressivity and generalizability; 2) ours can simplify the calculation process by unifying score forms and matrix aggregation; 3) ours can consider correlation differences in different layers using the layer-wise score with exponentially increasing weights from the bottom to the top.

\subsubsection{Proxy Voting}\label{section:voting}

Proxy voting refers to the process where multiple proxies collaborate using predefined aggregation methods to estimate the performance of neural architectures.
Given that no single proxy can consistently and robustly outperform all other proxies across all scenarios \cite{PPP}, proxy voting for architecture estimation is standard practice \cite{ZeroCost, PPP}.
To improve the performance and stability of proxies, we propose a voting proxy, named SJW, which employs a majority voting rule based on three highly competitive proxies: SynFlow, JacCor, and WRCor.
For comparison, we also propose two additional voting proxies. ZeroCost \cite{ZeroCost} is a voting proxy of SynFlow, JacCor, and SNIP; SPW is a voting proxy of SynFlow, PNorm, and WRCor, where PNorm is the most cost-effective proxy except for Params and FLOPs.
We do not use NASWOT and ZiCo as the voting experts because they have similar measures of the expressivity and generalizability of neural architectures. 

\subsection{Search Strategies}

To search for the optimal architectures, we adopt three search strategies: random search (R), reinforcement learning (RL) \cite{NAS}, and regularized evolution (RE) \cite{AmoebaNet}. 
Random search directly scores a set of $N$ randomly sampled architectures.
As described in Algorithm~\ref{alg: RL}, our algorithm based on RL scores a sampled architecture $\alpha$ by the controller $\pi^\theta$. Then it uses the normalized score as the controller reward $R$ and updates the parameters $\theta$ of the controller by policy gradient $R \frac{\partial\log p(\alpha|\pi^\theta)}{\partial\theta}$ with extra momentum. The process iterates until $N$ architectures are explored.
As described in Algorithm~\ref{alg: RE}, our algorithm based on RE scores a set of $P$ randomly sampled architectures as the initial population. Then it mutates a parent architecture obtained through the Tournament selection method with $S$ competitors for the next generation and scores the next generation. The oldest individual in the population is eliminated. The process iterates until $N$ architectures are explored.
For each search strategy, we always select the architecture with the highest score from $N$ explored architectures as the final discovered architecture.
\begin{algorithm}[!htb]
  \caption{Reinforcement Learning}
  \label{alg: RL}
\begin{algorithmic}
  \STATE $history \leftarrow \varnothing$
  \WHILE{$|history|<N$}
  \STATE $model.arch \leftarrow$ $controller$.GenerateArchitecture()
  \STATE $model.score \leftarrow$ GetScore($model.arch$)
  \STATE add $model$ to $history$
  \STATE $controller.reward \leftarrow$ Normalize($model.score$)
  \STATE $controller$.Update($controller.reward$)
  \ENDWHILE
  \STATE \textbf{return} highest-accuracy model in $history$
\end{algorithmic}
\end{algorithm}
\begin{algorithm}[!htb]
  \caption{Regularized Evolution}
  \label{alg: RE}
\begin{algorithmic}
  \STATE $population \leftarrow$ empty queue
  \STATE $history \leftarrow \varnothing$
  \WHILE{$|population|<P$}
  \STATE $model.arch \leftarrow$ RandomArchitecture()
  \STATE $model.score \leftarrow$ GetScore($model.arch$)
  \STATE add $model$ to right of $population$
  \STATE add $model$ to $history$
  \ENDWHILE
  \WHILE{$|history|<N-P$}
  \STATE $sample \leftarrow \varnothing$
  \WHILE{$|sample|<S$}
  \STATE $candidate \leftarrow$ random element from $population$
  \STATE add $candidate$ to $sample$
  \ENDWHILE
  \STATE $parent \leftarrow$ highest-accuracy model in $sample$
  \STATE $child.arch \leftarrow$ Mutate($parent.arch$)
  \STATE $child.score \leftarrow$ GetScore($child.arch$)
  \STATE add $child$ to right of $population$
  \STATE add $child$ to $history$
  \STATE remove $dead$ from left of $population$
  \STATE discard $dead$
  \ENDWHILE
  \STATE \textbf{return} highest-accuracy model in $history$
\end{algorithmic}
\end{algorithm}

\section{Experiments}

\subsection{Implementation Details}

\subsubsection{Search Space and Benchmark}

We validate our training-free proxies and NAS algorithms on three search spaces and benchmarks.

NAS-Bench-101 search space \cite{NASBench101} contains 423k unique convolutional architectures.
The NAS-Bench-101 architectures are composed of a stem layer followed by stacks of cells. Each cell is stacked 3 times, followed by a downsampling layer. This pattern is repeated 3 times, followed by global average pooling and a final dense softmax layer. Cells are represented as directed acyclic graphs (DAGs) with $V \le 7$ nodes, $L=3$ candidate operations ($3\times 3$ convolution, $1\times 1$ convolution, and $3\times 3$ max-pool), and $E \le 9$ edges.
NAS-Bench-101 benchmark provides the evaluation results of three runs on the CIFAR-10 \cite{CIFAR} dataset for its 423k architectures.

NAS-Bench-201 search space \cite{NASBench201} consists of 15,625 possible architectures in total.
The NAS-Bench-201 architectures follow the similar design of its counterpart as used in the NAS-Bench-101 benchmark. Each cell is represented as a densely connected DAG with $V=4$ nodes. The DAG is obtained by assigning a direction from the $i$-th node to the $j$-th node ($i < j$) for each edge in an undirected complete graph. Different NAS-Bench-101, their edges are associated with one of $L=5$ candidate operations (zeroize, skip connection, $3\times 3$ convolution, $1\times 1$ convolution, and $3\times 3$ average-pool).
NAS-Bench-201 benchmark gives the evaluation results on the CIFAR-10 \cite{CIFAR}, CIFAR-100 \cite{CIFAR}, and ImageNet16-120 \cite{DownsampledImageNet} datasets for its 15,625 architectures. 

MobileNetV2 search space \cite{MobileNetV2, ZenNAS} is a popular open-domain search space. Each architecture consists of stacked MobileNet Blocks. In each block, we can search for the depth-wise expansion ratio (E), the kernel size (K), the stride (S), the number of layers (L), the number of channels (C), the number of bottleneck channels (B), and the resolution of input images (R). The details of our MobileNetV2 search space are shown in Table~\ref{tab: MBV2SPACE}.
\begin{table}
  \caption{The details of the MobileNetV2 search space. $[s, e, i]$ refers to the values with an interval of $i$ within the range $[s, e]$.}
  \label{tab: MBV2SPACE}
  \centering
  \begin{tabular}{c@{}c@{}c@{}c@{}c@{}c@{}c@{}c@{}c}
      \hline
      Stage & Op & E & K & S & L & C & B & R \\
      \hline
      - & Conv & -         & 3       & 2   & 1   & [16,48,8]    & -      & [192,320,32] \\
      1 & MB   & \{1,2,4,6\} & [3,7,2] & [1,2] & [1,2] & [16,48,8]    & [16,48,8]  & - \\
      2 & MB   & \{1,2,4,6\} & [3,7,2] & [1,2] & [2,3] & [16,48,8]    & [16,48,8]  & - \\
      3 & MB   & \{1,2,4,6\} & [3,7,2] & [1,2] & [2,3] & [32,64,8]    & [16,48,8]  & - \\
      4 & MB   & \{1,2,4,6\} & [3,7,2] & [1,2] & [2,4] & [48,96,8]    & [32,64,8]  & - \\
      5 & MB   & \{1,2,4,6\} & [3,7,2] & [1,2] & [2,6] & [80,256,8]   & [48,96,8]  & - \\
      6 & MB   & \{1,2,4,6\} & [3,7,2] & [1,2] & [2,6] & [80,256,8]   & [80,256,8] & - \\
      7 & MB   & \{1,2,4,6\} & [3,7,2] & [1,2] & [1,2] & [128,512,8]  & [80,256,8] & - \\
      - & Conv & -         & 1       & 1   & 1   & [128,1536,8] & -      & - \\
      \hline
  \end{tabular}
\end{table}

\subsubsection{Dataset and Data Augmentation}

We conduct our experiments on the CIFAR-10 \cite{CIFAR}, CIFAR-100 \cite{CIFAR}, ImageNet16-120 \cite{DownsampledImageNet}, and ImageNet-1k \cite{ImageNet} datasets, with their official training/validation/testing splits.
We implement the standard augmentations including random flip/crop and cutout for the CIFAR-10, CIFAR-100, and ImageNet16-120 datasets. We implement the following augmentations \cite{ZenNAS} for the ImageNet-1k dataset: random crop/resize/flip/lighting, AutoAugment, random erasing, mix-up, and label-smoothing.

\subsubsection{Proxy Evaluation}

We conduct the proxy evaluation experiments to prove the effectiveness of our training-free proxies. We use the Spearman ranking correlation coefficient (Spearman $\rho$) to quantify how well a proxy ranks architectures compared to the ground-truth ranking through training. For efficiency, we use 1000 randomly sampled architectures from the NAS-Bench-201 search space to estimate the real Spearman $\rho$ values of all architectures on the CIFAR-10 dataset for different proxies. We conduct multiple experiments across various settings, including different datasets (\textbf{CIFAR-10}, CIFAR-100, and ImageNet16-120 datasets), batch sizes (32, \textbf{64}, and 128), initializations (\textbf{Kaiming uniform}, Kaiming normal, and $N(0,1)$ initializations), and benchmarks (NAS-Bench-101 and \textbf{NAS-Bench-201}).
The default setting of our proxies is bold in our experiments. All experiments are conducted on a single Tesla V100 GPU.

\subsubsection{Architecture Search}

We use three promising training-free proxies (\textbf{ZeroCost}, \textbf{WRCor}, and \textbf{SJW}) as the estimation strategy of zero-shot NAS and conduct the architecture search experiments. We use three search strategies for zero-shot NAS, including random search, reinforcement learning, and evolutionary algorithm. We explore $N=1000$ architectures in NAS-Bench-101/201 search space for fair comparison with the results reported in NAS-Bench-201 \cite{NASBench201}. We also explore $N=5000$ architectures in MobileNetV2 search space that take four GPU hours. For reinforcement learning, we use an Adam optimizer with a learning rate of $\alpha=0.001$ and a momentum of 0.9 for the exponential moving average. We use the regularized evolution algorithm with a population size of $P=64$ and a tournament size of $S=10$. All experiments are conducted on a single Tesla V100 GPU.

\subsubsection{Architecture Evaluation}

To verify the effectiveness of our NAS algorithms, we evaluate the discovered NAS-Bench-101/201 architectures by looking up the architecture performance table of the corresponding benchmark and evaluate the discovered MobileNetV2 architectures by standard training on the ImageNet-1k dataset. We use an SGD optimizer with a momentum of 0.9 and a weight decay of 4e-5. Its initial learning rate is 0.2 and is scheduled by cosine decay. We use ResNet-152 as a teacher network and train the MobileNetV2 architectures with a batch size of 512 and 480 epochs. All experiments are conducted on four Tesla V100 GPUs.

\subsection{Analysis of Proxy Evaluation}

Table~\ref{tab: realVSextimated} reports the real and estimated Spearman $\rho$ values of all proxies. The deviations between the real and estimated Spearman $\rho$ values are within 6\% for most proxies, which can be acceptable for rapid proxy evaluation and ablation studies. We also report and analyze their Spearman $\rho$ values across different settings. All reported Spearman $\rho$ values in our paper are their absolute values.
\begin{table}
  \caption{The comparison of the real/estimated Spearman $\rho$ values for different proxies on the CIFAR-10 dataset and the NAS-Bench-201 benchmark. The real and estimated Spearman $\rho$ values are calculated by all architectures and 1000 randomly sampled architectures of the NAS-Bench-201 benchmark, respectively. $\Delta$: the deviation between the real and estimated Spearman $\rho$ values; \%: the deviation ratio. ValAcc: the validation accuracy through standard training. ValACC is a real performance metric and thus is the theoretical upper bound.}
  \label{tab: realVSextimated}
  \centering
  \begin{tabular}{lrrrr}
    \hline
    Proxies & est.  & real  & $\Delta$  & \%  \\
    \hline
    PNorm   & 0.684 & 0.684 & 0.000   & 0.00  \\
    PGNorm  & 0.591 & 0.594 & \textcolor{green}{-0.003}  & \textcolor{green}{-0.51} \\
    SNIP    & 0.593 & 0.596 & \textcolor{green}{-0.003}  & \textcolor{green}{-0.50} \\
    GraSP   & 0.543 & 0.515 & \textcolor{red}{0.028}   & \textcolor{red}{5.44}  \\
    SynFlow & 0.743 & 0.739 & \textcolor{red}{0.004}   & \textcolor{red}{0.54}  \\
    Fisher  & 0.506 & 0.503 & \textcolor{red}{0.003}   & \textcolor{red}{0.60}  \\
    JacCor  & 0.726 & 0.726 & 0.000   & 0.00  \\
    EPE     & 0.693 & 0.694 & \textcolor{green}{-0.001}  & \textcolor{green}{-0.14} \\
    Phi     & 0.703 & 0.723 & \textcolor{green}{-0.020}  & \textcolor{green}{-2.77} \\
    Zen     & 0.257 & 0.237 & \textcolor{red}{0.020}   & \textcolor{red}{8.44}  \\
    BT      & 0.510 & 0.472 & \textcolor{red}{0.038}   & \textcolor{red}{8.05}  \\
    BndlEnt & 0.199 & 0.188 & \textcolor{red}{0.011}   & \textcolor{red}{5.85}  \\
    LyDynIso& 0.633 & 0.632 & \textcolor{red}{0.001}   & \textcolor{red}{0.16}  \\
    DynIso  & 0.662 & 0.674 & \textcolor{green}{-0.012}  & \textcolor{green}{-1.78} \\
    NetSens & 0.686 & 0.685 & \textcolor{red}{0.001}   & \textcolor{red}{0.15}  \\
    NASWOT  & 0.779 & 0.780 & \textcolor{green}{-0.001}  & \textcolor{green}{-0.13} \\
    ZiCo    & -     & 0.800 & -       & -     \\
    Params  & -     & 0.750 & -       & -     \\
    FLOPs   & -     & 0.730 & -       & -     \\
    \hline
    ValAcc  & 0.842 & 0.811 & \textcolor{red}{0.031}   & \textcolor{red}{3.82}  \\
    \hline
    ACor (Ours)    & 0.787 & 0.789 & \textcolor{green}{-0.002} & \textcolor{green}{-0.25} \\
    RCor (Ours)    & 0.803 & 0.803 & \textbf{0.000} & \textbf{0.00} \\
    WRCor (Ours)   & \textbf{0.811} & \textbf{0.812} & \textcolor{green}{-0.001} & \textcolor{green}{-0.12} \\
    \hline
    ZeroCost& 0.775 & 0.750 & \textcolor{red}{0.025} & \textcolor{red}{3.33} \\
    SPW (Ours)     & 0.789 & 0.780 & \textbf{\textcolor{red}{0.009}} & \textbf{\textcolor{red}{1.15}} \\
    SJW (Ours)     & \textbf{0.851} & \textbf{0.799} & \textcolor{red}{0.052} & \textcolor{red}{6.51} \\
    \hline
  \end{tabular}
\end{table}

\paragraph{Distinction of Bottom/Top-Level Features}

Through experiments, we find that for superior architectures, the correlations of responses in their top layers are greater than those in their bottom layers. Conversely, for inferior architectures, the correlations in their top layers are worse than those in their bottom layers. For example, for the best NAS-Bench-101 architecture, the correlation scores in its first and final layers are -685.21 and -506.27 respectively; for the worst NAS-Bench-101 architecture, the scores in its first and final layers are -538.55 and -1098.02 respectively.
The scores of their final layers are positively correlated with their accuracies, which is consistent with our conclusion. The scores of their first layers are negatively correlated with their accuracies. This is not a common phenomenon because it is an extreme case where we are comparing the best and worst architectures.
Based on these observations, we propose a layer-wise weighting score, i.e., WRCor.

\paragraph{Different Ways of Weighting Correlation of Responses}

To find a superior weighting way, we experiment with different ways of weighting the correlation of responses, including no weighting (no), weighting linearly (linear), quadratically (quad), and exponentially (expb2t), with weights increasing either from the bottom to the top layers or vice versa (expt2b). As shown in Table~\ref{tab: weighting}, the expb2t weighting way can achieve a notable improvement compared to linearly and quadratically weighting ways. WRCor with the expb2t weighting has 0.008 (1.00\%) higher Spearman $\rho$ value than RCor without weighting. The expt2b weighting harms architecture ranking and has 0.003 (0.37\%) lower Spearman $\rho$ value than RCor. Experimental results demonstrate that the correlations of responses in the top layers are more significant than those in the bottom layers. The expb2t weighting adopted by our WRCor proxy outperforms other weighting ways.
Besides, the above case of best and worst NAS-Bench-101 architectures also indicate that we should pay more attention the correlations of responses in the top layers and pay less attention or even ignore thoes in the bottom layers. Therefore, we adopt exponential weighting.
\begin{table}
  \caption{The comparison of different weighting ways for correlation of responses on the CIFAR-10 dataset and the NAS-Bench-201 benchmark. $\Delta$: the deviation between RCor and the proxy; \%: the deviation ratio.}
  \label{tab: weighting}
  \centering
  \begin{tabular}{lrrr}
    \hline
    Proxies   & $\rho$  & $\Delta$ & \%   \\
    \hline
    no (RCor) & 0.803   & 0.000   & 0.00  \\
    linear    & 0.805   & \textcolor{red}{0.002}   & \textcolor{red}{0.25}  \\
    quad & 0.806   & \textcolor{red}{0.003}   & \textcolor{red}{0.37}  \\
    expb2t (WRCor) & \textbf{0.811}   & \textbf{\textcolor{red}{0.008}}   & \textbf{\textcolor{red}{1.00}}  \\
    expt2b & 0.800   & \textcolor{green}{-0.003}  & \textcolor{green}{-0.37} \\
    \hline
  \end{tabular}
\end{table}

\paragraph{How Do Our Proposed Proxies Perform?}

Table~\ref{tab: realVSextimated} reports the Spearman $\rho$ values of ACor, RCor, and WRCor. ACor significantly outperforms most existing proxies, including Params and FLOPs. But it performs marginally better than NASWOT and falls short of ZiCo. RCor and WRCor achieve improvements of 0.014 (1.77\%) and 0.023 (2.92\%) higher Spearman $\rho$ values than ACor, respectively. WRCor surpasses ZiCo with 0.012 (1.5\%) higher Spearman $\rho$ value. Remarkably, a single WRCor proxy outperforms even two voting proxies (i.e., ZeroCost and SPW) and is close to ValACC in terms of the real Spearman $\rho$ value. WRCor is more efficient than standard training, requiring approximately two seconds to estimate each architecture. Therefore, WRCor is an accurate and efficient training-free proxy. Further analyses are as follows.

\paragraph{Visualization of Different Proxies}

Figure~\ref{fig: vsTestAcc} demonstrates the relationships between different estimated proxy metrics and the real test accuracy. Figure~\ref{fig: SynFlow} and \ref{fig: JacCor} demonstrate that the skewed distributions of SynFlow and JacCor proxies are not suitable for ranking these architectures and are the main reason for the low Spearman correlation. As expected, Figure~\ref{fig: Params} and Figure~\ref{fig: FLOPs} show that the Params and FLOPs proxies exhibit excellent performance within each score segment, but it does not perform as well across all score segments collectively. Figure~\ref{fig: ACor} and \ref{fig: RCor} show Acor and Rcor have a stronger correlation with the real test accuracy. Specifically, with the help of the correlation of gradients, the distribution in Figure~\ref{fig: RCor} is approximately consistent with that of Figure~\ref{fig: ValAcc}. As shown in Figure~\ref{fig: WRCor}, WRCor exhibits a more compact distribution by weighting the correlation of responses, which demonstrates that WRCor has a higher Spearman ranking correlation with the real test accuracy. 
\begin{figure*}
  \centering
  \begin{subfigure}{.24\textwidth}
    \centering
    \includegraphics[width=\linewidth]{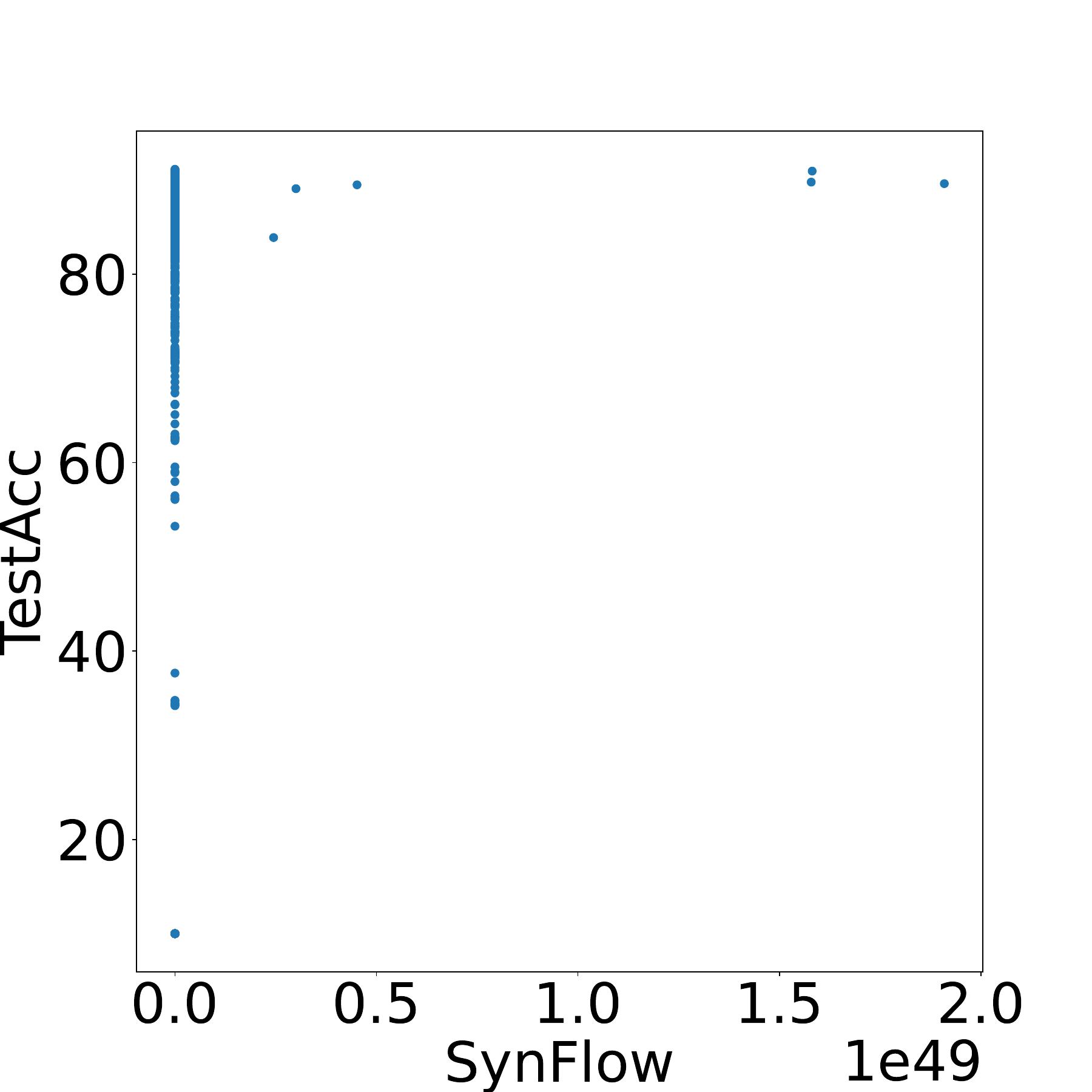}
    \caption{SynFlow ($\rho$: 0.743)}\label{fig: SynFlow}
  \end{subfigure}
  \begin{subfigure}{.24\textwidth}
    \centering
    \includegraphics[width=\linewidth]{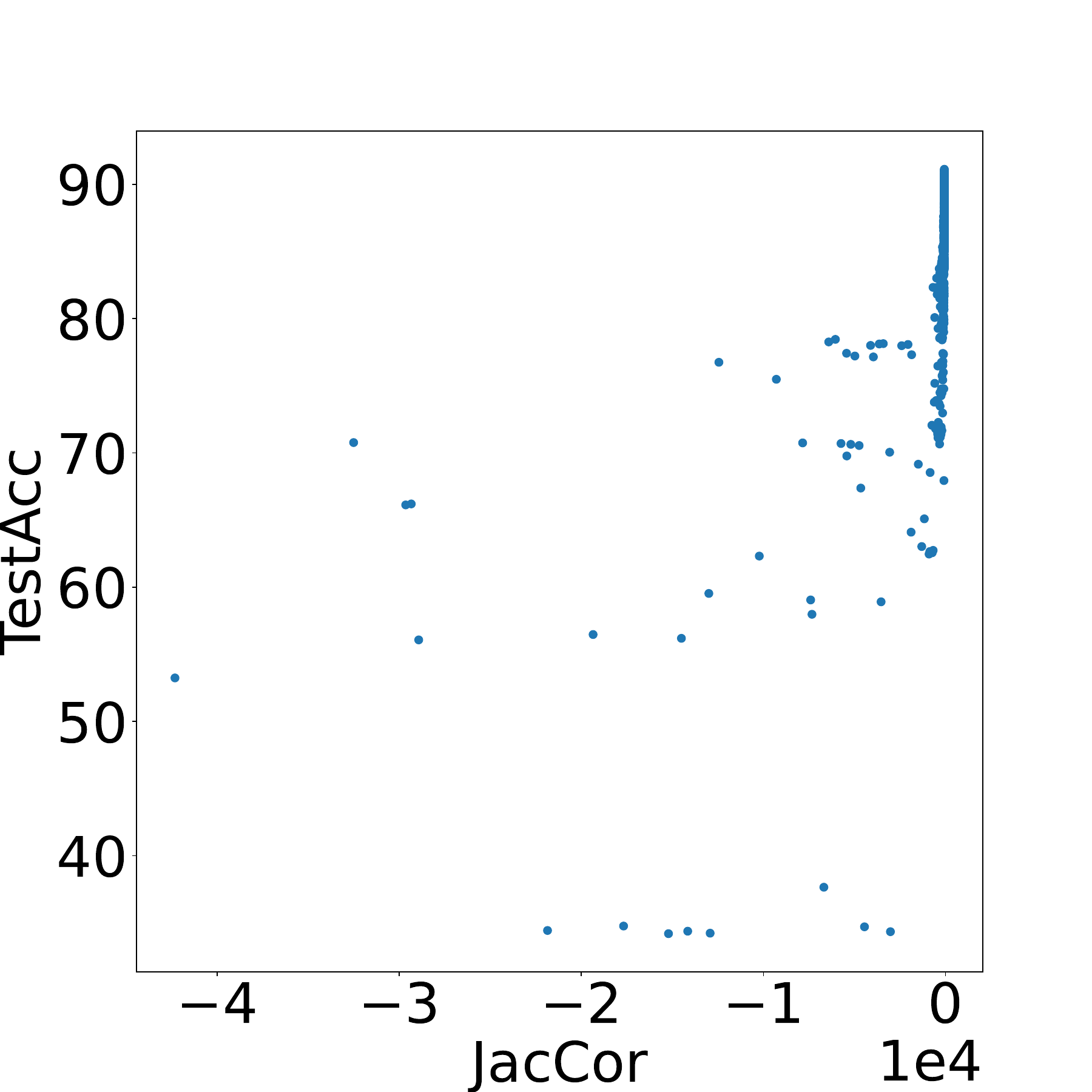}
    \caption{JacCor ($\rho$: 0.726)}\label{fig: JacCor}
  \end{subfigure}
  \begin{subfigure}{.24\textwidth}
    \centering
    \includegraphics[width=\linewidth]{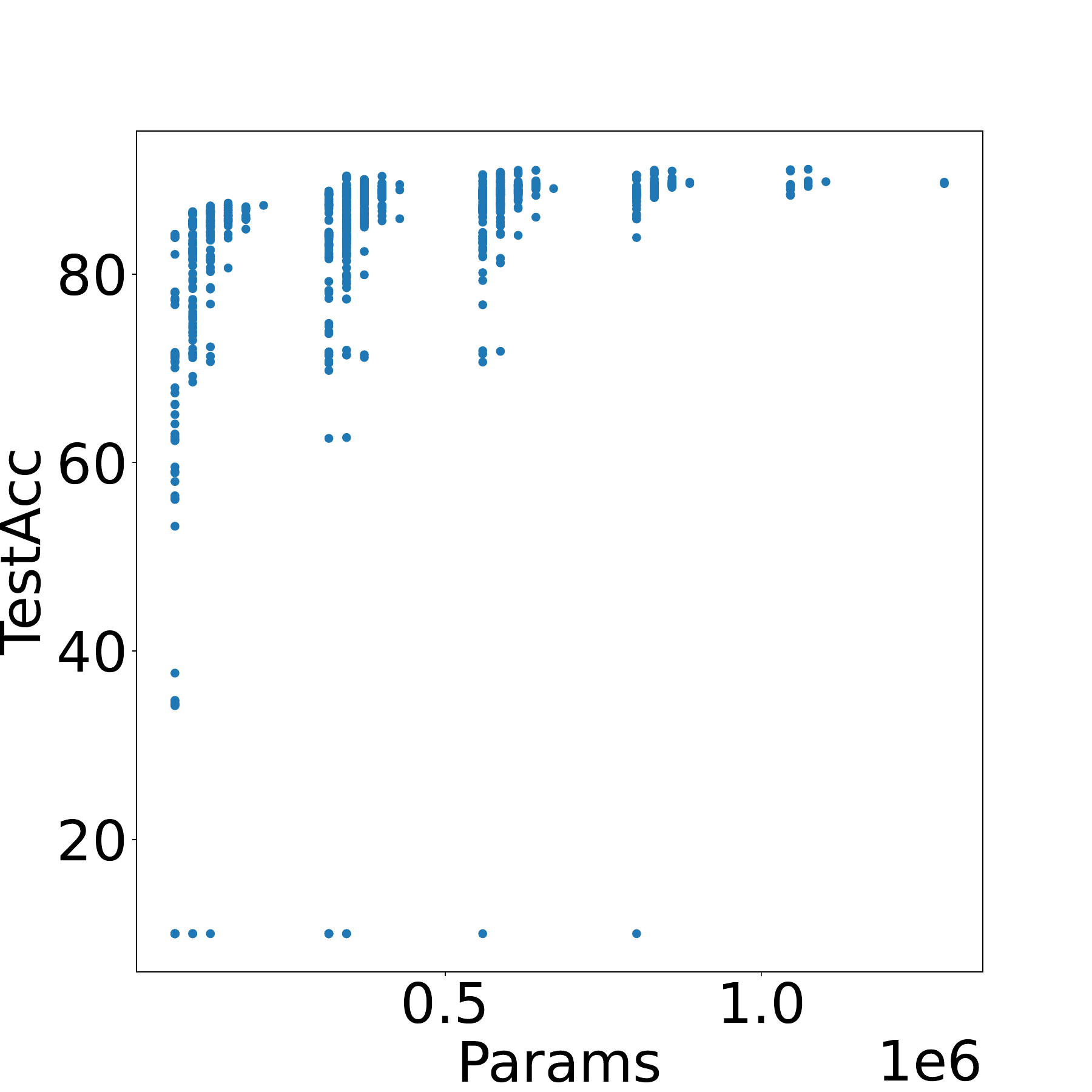}
    \caption{Params ($\rho$: 0.750)}\label{fig: Params}
  \end{subfigure}
  \begin{subfigure}{.24\textwidth}
    \centering
    \includegraphics[width=\linewidth]{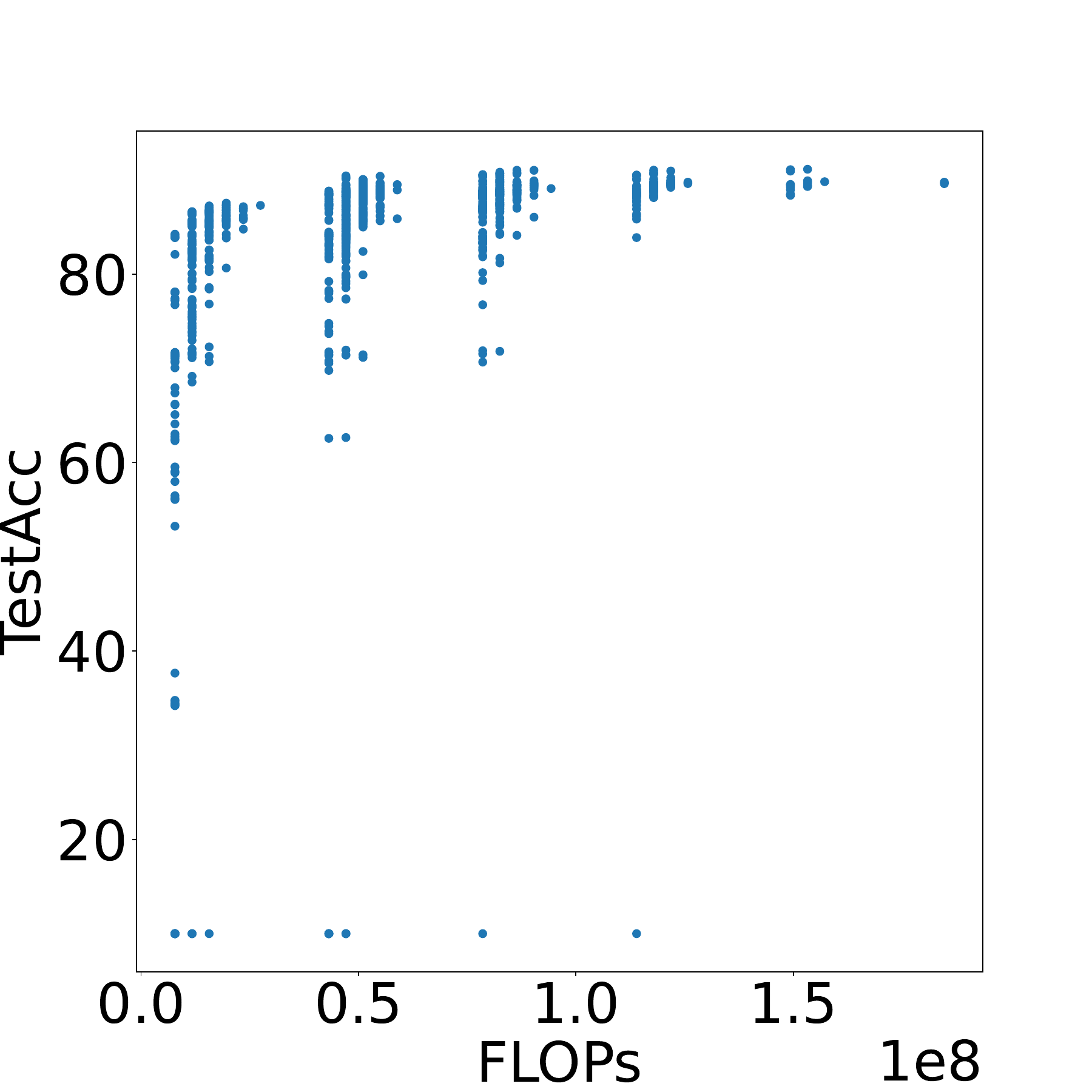}
    \caption{FLOPs ($\rho$: 0.730)}\label{fig: FLOPs}
  \end{subfigure}
  
  \begin{subfigure}{.24\textwidth}
    \centering
    \includegraphics[width=\linewidth]{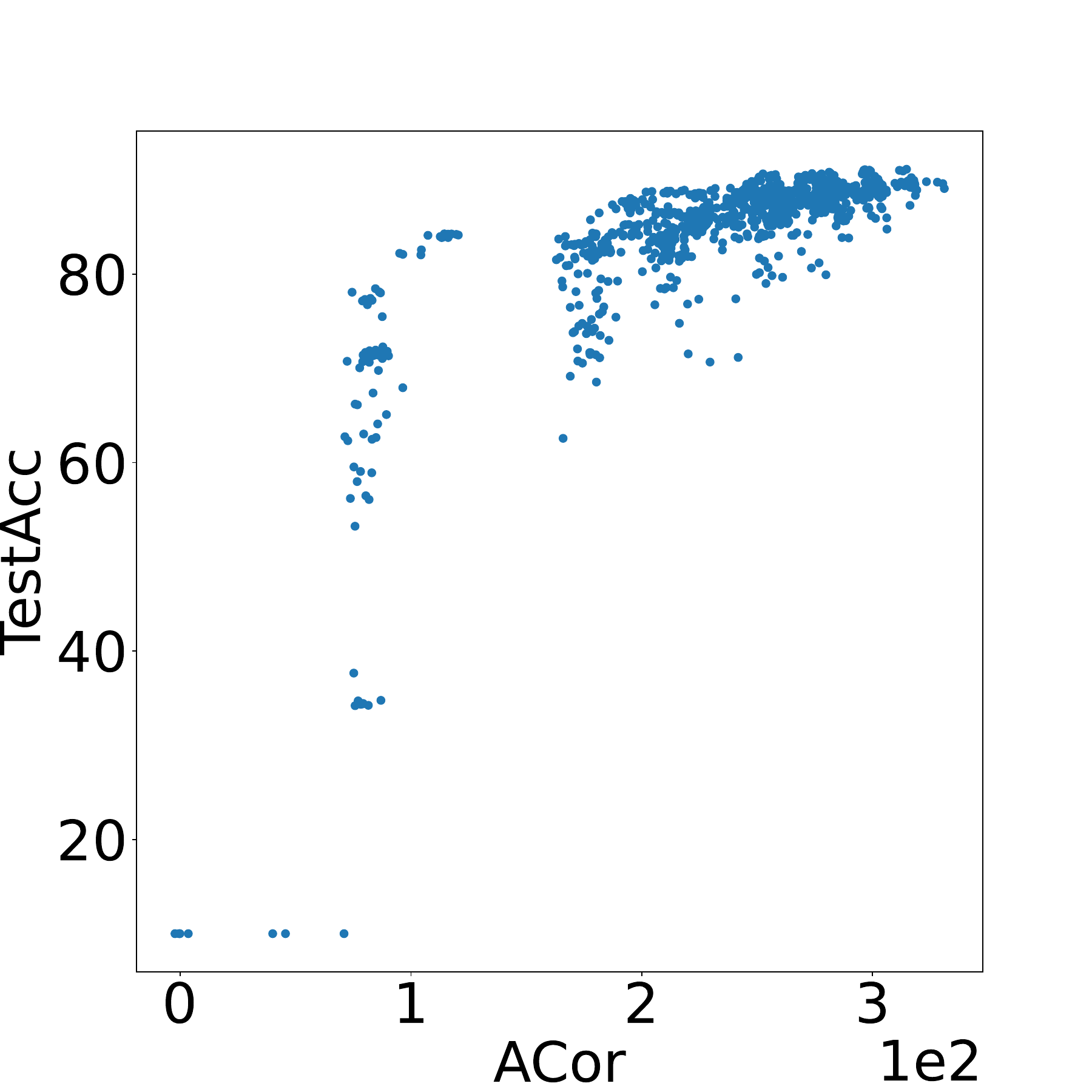}
    \caption{ACor ($\rho$: 0.787)}\label{fig: ACor}
  \end{subfigure}
  \begin{subfigure}{.24\textwidth}
    \centering
    \includegraphics[width=\linewidth]{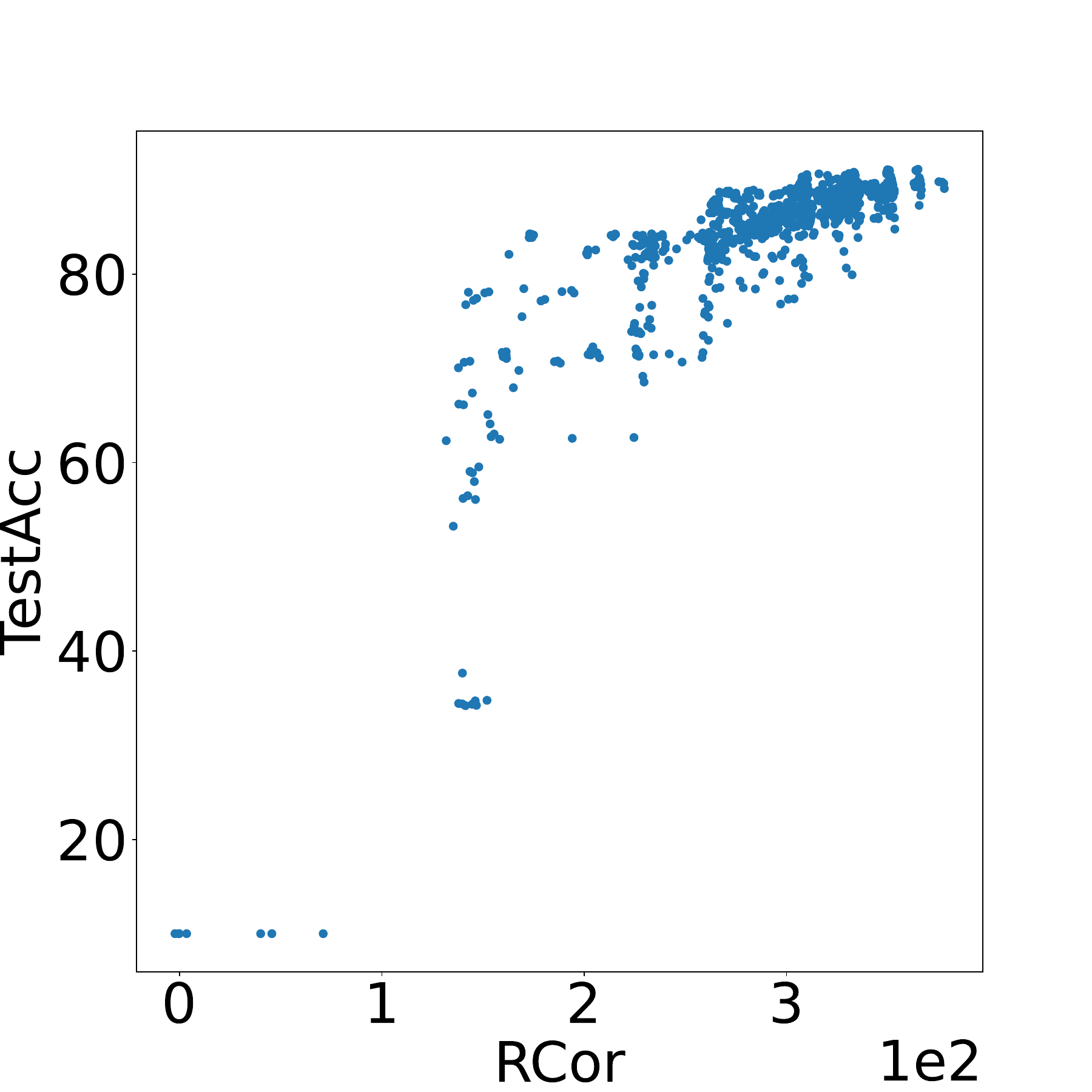}
    \caption{RCor ($\rho$: 0.803)}\label{fig: RCor}
  \end{subfigure}
  \begin{subfigure}{.24\textwidth}
    \centering
    \includegraphics[width=\linewidth]{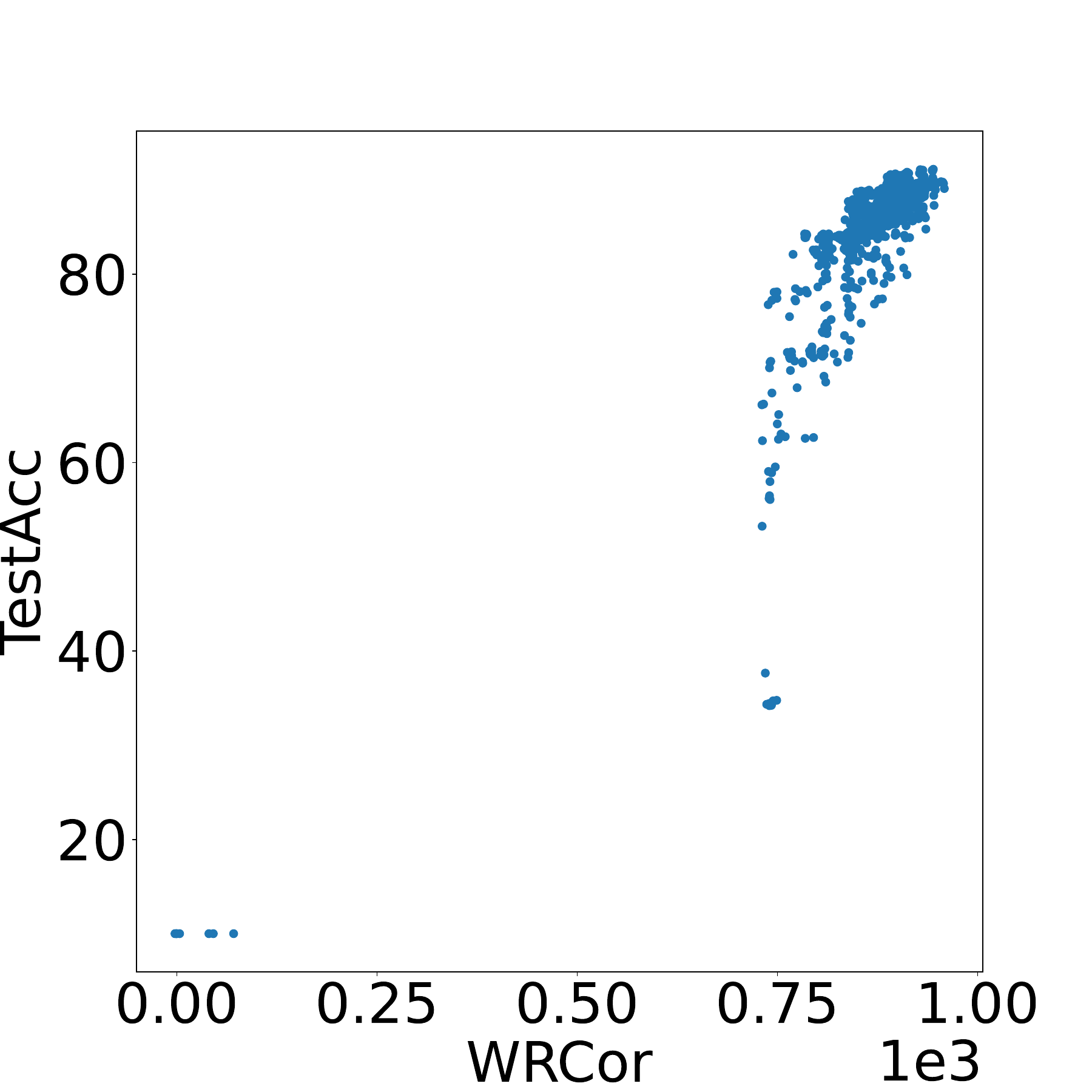}
    \caption{WRCor ($\rho$: 0.811)}\label{fig: WRCor}
  \end{subfigure}
  \begin{subfigure}{.24\textwidth}
    \centering
    \includegraphics[width=\linewidth]{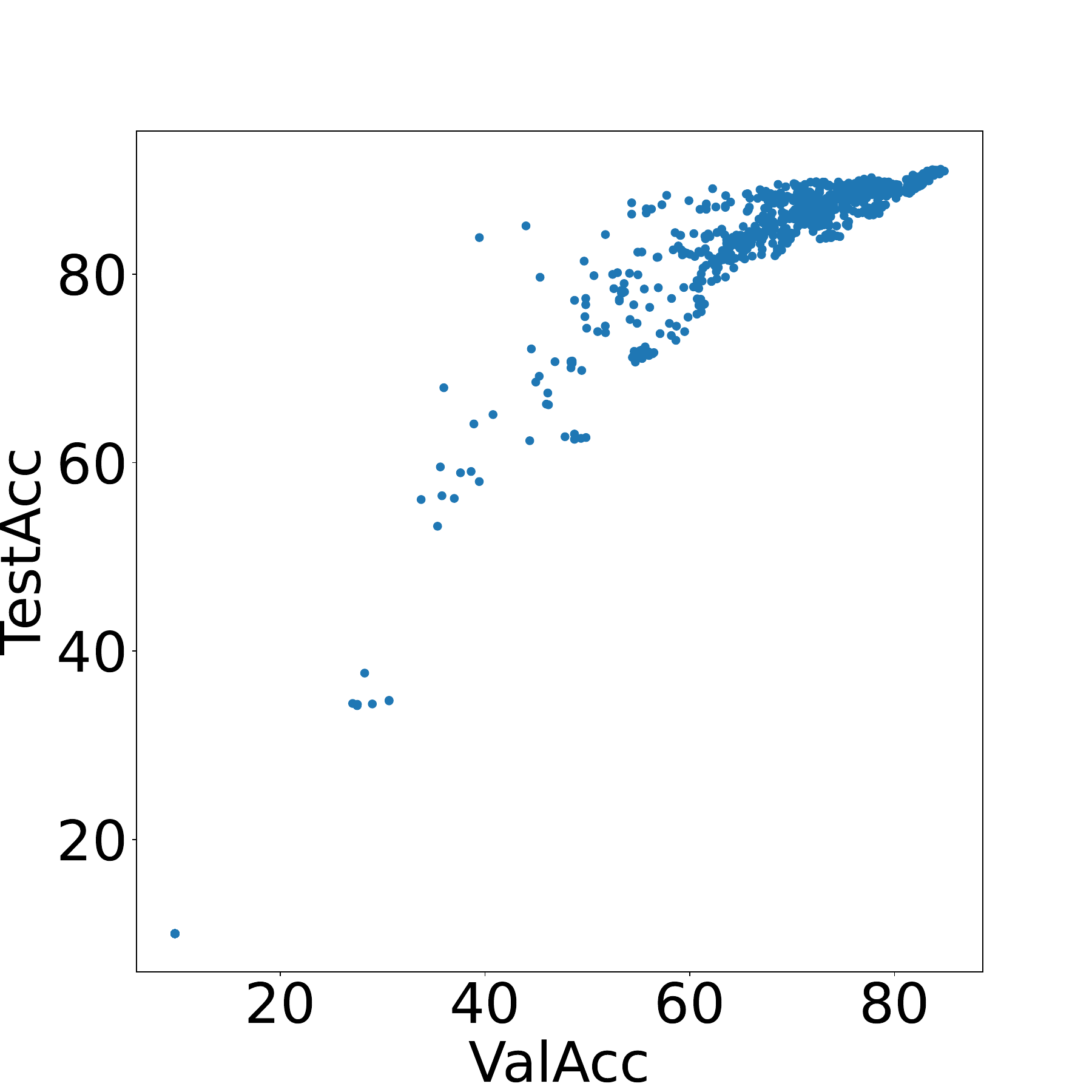}
    \caption{ValAcc ($\rho$: 0.842)}\label{fig: ValAcc}
  \end{subfigure}
  \caption{The visualization of different proxies and test accuracy on the CIFAR-10 dataset by 1000 random NAS-Bench-201 architectures.}\label{fig: vsTestAcc}
\end{figure*}

\paragraph{How Do They Perform across Different Datasets?}

Table~\ref{tab: datasetVS} reports the Spearman $\rho$ values of all mentioned proxies across the CIFAR-10, CIFAR-100, and ImageNet16-120 datasets. Our proxies are superior to existing proxies in terms of the average and single Spearman $\rho$ values across three different datasets. Specifically, the average Spearman $\rho$ value of WRCor is 0.026 (3.25\%) higher than that of the state-of-the-art ZiCo. Furthermore, WRCor is comparable to or even surpasses ValACC on CIFAR-100 and ImageNet16-120 datasets. Besides, our proposed ACor, RCor, and WRCor proxies are also more robust than most existing proxies with a low standard deviation of 0.016, 0.015, and 0.015. Therefore, our proposed proxies can perform better than existing proxies across different datasets and are insensitive to datasets.
\begin{table}
  \caption{The comparison of different proxies across the CIFAR-10 (c10), CIFAR-100 (c100), and ImageNet16-120 (im120) datasets on the NAS-Bench-201 benchmark.}
  \label{tab: datasetVS}
  \centering
  \begin{tabular}{lrrrrrr}
    \hline
    Proxies & c10   & c100  & im120 & avg.  & std.  \\
    \hline
    PNorm   & 0.684 & 0.712 & 0.711 & 0.702 & 0.013 \\
    PGNorm  & 0.591 & 0.648 & 0.588 & 0.609 & 0.028 \\
    SNIP    & 0.593 & 0.648 & 0.588 & 0.610 & 0.027 \\
    GraSP   & 0.543 & 0.555 & 0.551 & 0.550 & 0.005 \\
    SynFlow & 0.743 & 0.771 & 0.755 & 0.756 & 0.011 \\
    Fisher  & 0.506 & 0.563 & 0.501 & 0.523 & 0.028 \\
    JacCor  & 0.726 & 0.721 & 0.716 & 0.721 & \textbf{0.004} \\
    EPE     & 0.693 & 0.602 & 0.566 & 0.620 & 0.053 \\
    Phi     & 0.703 & 0.730 & 0.663 & 0.699 & 0.028 \\
    Zen     & 0.257 & 0.297 & 0.293 & 0.282 & 0.018 \\
    BT      & 0.510 & 0.449 & 0.460 & 0.473 & 0.027 \\
    BndlEnt & 0.199 & 0.192 & 0.130 & 0.174 & 0.031 \\
    LyDynIso& 0.633 & 0.682 & 0.668 & 0.661 & 0.021 \\
    DynIso  & 0.662 & 0.728 & 0.674 & 0.688 & 0.029 \\
    NetSens & 0.686 & 0.706 & 0.624 & 0.672 & 0.035 \\
    NASWOT  & 0.779 & 0.818 & 0.791 & 0.796 & 0.016 \\
    ZiCo    & 0.800 & 0.810 & 0.790 & 0.800 & 0.008 \\
    Params  & 0.750 & 0.730 & 0.690 & 0.723 & 0.025 \\
    FLOPs   & 0.730 & 0.710 & 0.670 & 0.703 & 0.025 \\
    \hline
    ValAcc  & 0.842 & 0.845 & 0.833 & 0.840 & 0.005 \\
    \hline
    ACor (Ours)    & 0.787 & 0.823 & 0.790  & 0.800 & 0.016 \\
    RCor (Ours)    & 0.803 & 0.839 & 0.817  & 0.820 & 0.015 \\
    WRCor (Ours)   & \textbf{0.811} & \textbf{0.846} & \textbf{0.822}  & \textbf{0.826} & 0.015 \\
    \hline
    ZeroCost       & 0.775 & 0.802 & 0.775  & 0.784 & 0.013 \\
    SPW (Ours)     & 0.789 & 0.823 & 0.808  & 0.807 & 0.014 \\
    SJW (Ours)     & \textbf{0.851} & \textbf{0.846} & \textbf{0.847}  & \textbf{0.848} & \textbf{0.002} \\
    \hline
  \end{tabular}
\end{table}

\paragraph{How Do They Perform across Different Batch Sizes?}

Table~\ref{tab: batchsizeVS} reports the Spearman $\rho$ values of all mentioned proxies across different batch sizes of the inputs, including 32, 64, and 128. Although RCor and WRCor exhibit a higher sensitivity to the factor of batch size, with a standard deviation of 0.009, compared to most existing proxies, it is noteworthy that their Spearman $\rho$ values positively correlate with batch size. This implies that given sufficient computing resources, they perform better when utilizing larger batch sizes. Furthermore, ACor, RCor, and WRCor outperform NASWOT by the higher average Spearman $\rho$ values of 0.004 (0.51\%), 0.024 (3.08\%), and 0.030 (3.85\%) than NASWOT across different batch sizes. In conclusion, our proxies demonstrate superior performance compared to existing proxies across different batch sizes and can benefit from the use of larger batch sizes.
\begin{table}
  \caption{The comparison of different proxies across different batch sizes of 32, 64, and 128 on the CIFAR-10 dataset and the NAS-Bench-201 benchmark.}
  \label{tab: batchsizeVS}
  \centering
  \begin{tabular}{lrrrrrr}
    \hline
    Proxies & 32    & 64    & 128   & avg.  & std.  \\
    \hline
    PNorm   & 0.684 & 0.684 & 0.684 & 0.684 & \textbf{0.000} \\
    PGNorm  & 0.597 & 0.591 & 0.586 & 0.591 & 0.004 \\
    SNIP    & 0.599 & 0.593 & 0.589 & 0.594 & 0.004 \\
    GraSP   & 0.522 & 0.543 & 0.519 & 0.528 & 0.011 \\
    SynFlow & 0.743 & 0.743 & 0.743 & 0.743 & \textbf{0.000} \\
    Fisher  & 0.510 & 0.506 & 0.507 & 0.508 & 0.002 \\
    JacCor  & 0.730 & 0.726 & 0.735 & 0.730 & 0.004 \\
    EPE     & 0.661 & 0.693 & 0.703 & 0.686 & 0.018 \\
    Phi     & 0.701 & 0.703 & 0.704 & 0.703 & 0.001 \\
    Zen     & 0.258 & 0.257 & 0.256 & 0.257 & 0.001 \\
    BT      & 0.466 & 0.510 & 0.541 & 0.506 & 0.031 \\
    BndlEnt & 0.185 & 0.199 & 0.228 & 0.204 & 0.018 \\
    LyDynIso& 0.633 & 0.633 & 0.633 & 0.633 & \textbf{0.000} \\
    DynIso  & 0.654 & 0.662 & 0.667 & 0.661 & 0.005 \\
    NetSens & 0.686 & 0.686 & 0.686 & 0.686 & \textbf{0.000} \\
    NASWOT  & 0.775 & 0.779 & 0.782 & 0.779 & 0.003 \\
    ZiCo    & -     & -     & -     & -     & -     \\
    Params  & 0.750 & 0.750 & 0.750 & 0.750 & \textbf{0.000} \\
    FLOPs   & 0.730 & 0.730 & 0.730 & 0.730 & \textbf{0.000} \\
    \hline
    ValAcc  & 0.842 & 0.842 & 0.842 & 0.842 & 0.000 \\
    \hline
    ACor (Ours)    & 0.781 & 0.787 & 0.780  & 0.783 & 0.003 \\
    RCor (Ours)    & 0.791 & 0.803 & 0.812  & 0.802 & 0.009 \\
    WRCor (Ours)   & \textbf{0.797} & \textbf{0.811} & \textbf{0.818}  & \textbf{0.809} & 0.009 \\
    \hline
    ZeroCost       & 0.783 & 0.775 & 0.768  & 0.775 & 0.006 \\
    SPW (Ours)     & 0.786 & 0.789 & 0.791  & 0.789 & \textbf{0.002} \\
    SJW (Ours)     & \textbf{0.829} & \textbf{0.851} & \textbf{0.814}  & \textbf{0.831} & 0.015 \\
    \hline
  \end{tabular}
\end{table}

\paragraph{How Do They Perform across Different Initializations?}

Table~\ref{tab: initializationVS} reports the Spearman $\rho$ values of all mentioned proxies across the Kaiming uniform \cite{KaimingInit}, Kaiming normal \cite{KaimingInit}, and the standard normal ($N(0,1)$) \cite{ZenNAS} initializations. The Spearman $\rho$ values of ACor, RCor, and WRCor exhibit remarkably low standard deviations across different initializations, specifically 0.009, 0.004, and 0.006, respectively. This demonstrates that our proposed proxies are robust and insensitive to initialization. Besides, RCor and WRCor have 0.020 (2.57\%) and 0.026 (3.34\%) higher average Spearman $\rho$ values than NASWOT across different initializations. Consequently, when utilizing our proposed proxies, there is no need to carefully choose the way to initialize the networks for higher Spearman $\rho$ values.
\begin{table}
  \caption{The comparison of different proxies across Kaiming uniform (uniform), Kaiming normal (normal), and $N(0,1)$ (n01) initializations on the CIFAR-10 dataset and the NAS-Bench-201 benchmark.}
  \label{tab: initializationVS}
  \centering
  \begin{tabular}{lrrrrrr}
    \hline
    Proxies &uniform& normal& n01   & avg.  & std.  \\
    \hline
    PNorm   & 0.684 & 0.692 & 0.725 & 0.700 & 0.018 \\
    PGNorm  & 0.591 & 0.559 & 0.369 & 0.506 & 0.098 \\
    SNIP    & 0.593 & 0.570 & 0.518 & 0.560 & 0.031 \\
    GraSP   & 0.543 & 0.451 & 0.106 & 0.367 & 0.188 \\
    SynFlow & 0.743 & 0.738 & 0.743 & 0.741 & 0.002 \\
    Fisher  & 0.506 & 0.480 & 0.453 & 0.480 & 0.022 \\
    JacCor  & 0.726 & 0.739 & 0.733 & 0.733 & 0.005 \\
    EPE     & 0.693 & 0.698 & 0.719 & 0.703 & 0.011 \\
    Phi     & 0.703 & 0.670 & 0.657 & 0.677 & 0.019 \\
    Zen     & 0.257 & 0.036 & 0.352 & 0.215 & 0.132 \\
    BT      & 0.510 & 0.558 & 0.568 & 0.545 & 0.025 \\
    BndlEnt & 0.199 & 0.028 & 0.040 & 0.089 & 0.078 \\
    LyDynIso& 0.633 & 0.612 & 0.717 & 0.654 & 0.045 \\
    DynIso  & 0.662 & 0.587 & 0.358 & 0.536 & 0.129 \\
    NetSens & 0.686 & 0.641 & 0.625 & 0.651 & 0.026 \\
    NASWOT  & 0.779 & 0.777 & 0.777 & 0.778 & 0.001 \\
    ZiCo    & -     & 0.800 & -     & -     & -     \\
    Params  & 0.750 & 0.750 & 0.750 & 0.750 & \textbf{0.000} \\
    FLOPs   & 0.730 & 0.730 & 0.730 & 0.730 & \textbf{0.000} \\
    \hline
    ValAcc  & 0.842 & 0.842 & 0.842 & 0.842 & 0.000 \\
    \hline
    ACor (Ours)    & 0.787 & 0.775 & 0.766  & 0.776 & 0.009 \\
    RCor (Ours)    & 0.803 & 0.798 & 0.794  & 0.798 & 0.004 \\
    WRCor (Ours)   & \textbf{0.811} & \textbf{0.805} & \textbf{0.796}  & \textbf{0.804} & 0.006 \\
    \hline
    ZeroCost       & 0.775 & 0.791 & 0.700  & 0.755 & 0.040 \\
    SPW (Ours)     & 0.789 & 0.780 & 0.804  & 0.791 & 0.010 \\
    SJW (Ours)     & \textbf{0.851} & \textbf{0.829} & \textbf{0.839}  & \textbf{0.840} & \textbf{0.009} \\
    \hline
  \end{tabular}
\end{table}

\paragraph{How Do They Perform across Different Search Spaces?}

Table~\ref{tab: searchspaceVS} reports the Spearman $\rho$ values of all mentioned proxies across the NAS-Bench-101 and NAS-Bench-201 search spaces. The Spearman $\rho$ values of all proxies on the NAS-Bench-101 search space are significantly lower than those on the NAS-Bench-201 search space. Especially, SNIP, Fisher, JacCor, and NetSens change from positive correlation to negative correlation. It is also observed that the Spearman $\rho$ value of ValACC is significantly reduced. This reduction could be due to the significantly larger size of the NAS-Bench-101 search space compared to NAS-Bench-201. Our proposed proxies demonstrate a smaller reduction of Spearman $\rho$ value from NAS-Bench-201 to NAS-Bench-101 than most existing proxies. Furthermore, ACor, RCor, and WRCor exhibit Spearman $\rho$ values that are 0.126 (30.73\%), 0.143 (34.88\%), and 0.124 (30.24\%) higher than NASWOT on the NAS-Bench-101 search space, respectively. But they perform slightly worse than PNorm on the NAS-Bench-101 search space. PNorm maintains a smaller reduction of Spearman $\rho$ value due to its direct relationship with parameters and their initial distributions, which may allow its better generalization on the search spaces of different sizes. Additionally, while WRCor performs slightly worse than ACor and RCor on the NAS-Bench-101 search space, it still outperforms most existing proxies. Overall, our proposed proxies are still robust and powerful even in the larger NAS-Bench-101 search space.
\begin{table}
  \caption{The comparison of different proxies across NAS-Bench-101 (NB101) and NAS-Bench-201 (NB201) search spaces on the CIFAR-10 dataset. $|\Delta|$: the absolute deviation of the Spearman $\rho$ values.}
  \label{tab: searchspaceVS}
  \centering
  \begin{tabular}{lrrr}
    \hline
    Proxies & NB101 & NB201 & $|\Delta|$  \\
    \hline
    PNorm   & \textbf{0.561} & 0.684 & 0.123 \\
    PGNorm  & 0.248 & 0.591 & 0.343 \\
    SNIP    & 0.157 & 0.593 & 0.436 \\
    GraSP   & 0.357 & 0.543 & 0.186 \\
    SynFlow & 0.429 & 0.743 & 0.314 \\
    Fisher  & 0.281 & 0.506 & 0.225 \\
    JacCor  & 0.362 & 0.726 & 0.364 \\
    EPE     & 0.008 & 0.693 & 0.685 \\
    Phi     & 0.016 & 0.703 & 0.687 \\
    Zen     & 0.200 & 0.257 & \textbf{0.057} \\
    BT      & 0.054 & 0.510 & 0.456 \\
    BndlEnt & 0.012 & 0.199 & 0.187 \\
    LyDynIso& 0.486 & 0.633 & 0.147 \\
    DynIso  & 0.281 & 0.662 & 0.381 \\
    NetSens & 0.302 & 0.686 & 0.384 \\
    NASWOT  & 0.410 & 0.779 & 0.369 \\
    ZiCo    & -     & 0.800 & -     \\
    Params  & 0.440 & 0.750 & 0.310 \\
    FLOPs   & 0.430 & 0.730 & 0.300 \\
    \hline
    ValAcc  & 0.647 & 0.842 & 0.195 \\
    \hline
    ACor (Ours)    & 0.536 & 0.787 & 0.251 \\
    RCor (Ours)    & 0.553 & 0.803 & 0.250 \\
    WRCor (Ours)   & 0.534 & \textbf{0.811} & 0.277 \\
    \hline
    ZeroCost       & 0.000 & 0.775 & 0.775 \\
    SPW (Ours)     & \textbf{0.569} & 0.789 & \textbf{0.220} \\
    SJW (Ours)     & 0.383 & \textbf{0.851} & 0.468 \\
    \hline
  \end{tabular}
\end{table}

\paragraph{Does the input of our proposed proxy require relevance to the target dataset?}

Instead of utilizing images from datasets, we compute our proxies using noises sampled from the standard normal distribution as the input. WRCor with noises can achieve the Spearman $\rho$ value of 0.784, which is 0.027 (3.33\%) lower than WRCor with the CIFAR-10 dataset. The experimental result demonstrates that even without the use of relevant datasets, our proxies can still achieve the 0.005 (0.64\%) higher Spearman $\rho$ value than NASWOT. Therefore, our proposed proxies, despite not relying on relevant datasets, still exhibit competitive performance compared to other existing proxies, and the performance of our proxies can be further enhanced using the relevant datasets.

\paragraph{How Do Our Proposed Voting Proxies Perform?}

Some NAS works \cite{ZeroCost, PPP} demonstrate that we can achieve an omnipotent architecture performance predictor by integrating multiple proxies and predictors. To capitalize on the strengths of different proxies to improve their Spearman $\rho$ values for the above challenging scenarios, such as
exploring the large search space, we employ a majority voting rule among various proxies for architecture estimation.
Although all mentioned voting proxies rely on SynFlow, there are many differences between them, as shown in Table~\ref{tab: realVSextimated}, \ref{tab: datasetVS}, \ref{tab: batchsizeVS}, \ref{tab: initializationVS}, and \ref{tab: searchspaceVS}.
SPW and SJW that utilize WRCor usually have significantly higher Spearman $\rho$ value, which demonstrates that WRCor is a crucial and effective proxy. Since JacCor performs better than PNorm, SJW outperforms SPW in most scenarios, except when exploring the larger NAS-Bench-101 search space. The reason for the performance reduction of SJW in the larger NAS-Bench-101 search space is attributed to the fact that JacCor \cite{NASWOT} estimates networks using the histograms ($0<(\Sigma_J)_{i,j}<\beta$) of the correlation coefficient matrices. This is a coarse-grained score and can lead to similar architectures receiving the same scores as the search space expands unless the upper bound $\beta$ is appropriately adjusted. SPW is a better proxy when searching in the larger search space.
In summary, SPW and SJW compensate for the shortcomings of our proposed proxies
, which can ensure a more accurate architecture evaluation.

\subsection{Analysis of Architecture Search and Evaluation}

\subsubsection{Search on NAS-Bench-101 Search Space}

To compare our zero-shot NAS algorithms across various proxies and search strategies, we conduct experiments on the NAS-Bench-101 benchmark. Table~\ref{tab: NASBENCH101} reports their validation/test accuracies and search costs on the CIFAR-10 dataset. Compared with X-ZeroCost, X-WRCor promotes the test accuracies by $1.99\sim8.81$ with higher robustness. RL-WRCor and R-WRCor discover the second and third competitive architectures with the test accuracies of 93.36\% and 93.26\%, respectively. X-SJW slightly outperforms X-WRCor. Since X-SJW replaces SNIP \cite{SNIP} with WRCor for voting, it has the $2.10\sim8.86$ higher test accuracies with higher robustness than X-ZeroCost. Notably, RL-SJW discovers the most competitive architectures with an average test accuracy of 93.47\% and a minimal standard deviation of 0.03. Furthermore, X-WRCor and X-SJW discover architectures that are comparable to those discovered by X-ValACC, but X-ValACC requires a $25\sim50$ higher cost. Our proposed WRCor and SJW proxies are effective, efficient, and robust estimation strategies for NAS.
\begin{table}
  \caption{The comparison of zero-shot NAS across different search strategies and training-free proxies on the CIFAR-10 dataset and the NAS-Bench-101 benchmark. We run each algorithm three times where $N=1000$. $^\star$, $^\ddagger$, and $^\dagger$ represent the top-3 algorithms in the last three blocks. X-Y represents the zero-shot NAS algorithm that uses search strategy X and training-free proxy Y.}
  \label{tab: NASBENCH101}
  \centering
  \begin{tabular}{l@{}rll}
      \hline
      &&\multicolumn{2}{c}{CIFAR-10 Acc (\%)}\\
      \cmidrule(r){3-4}
      Algorithm & Cost (s) & Validation & Test \\
      \hline
      R-ValAcc & 142980 & 93.92 $\pm$ 0.11 & 93.42 $\pm$ 0.02 \\
      RL-ValAcc & 157375 & 94.23 $\pm$ 0.33 & 93.72 $\pm$ 0.27 \\
      RE-ValAcc & 152146 & 94.28 $\pm$ 0.24 & 93.78 $\pm$ 0.16 \\
      \hline
      Optimum & - & 95.06 & 94.32 \\
      \hline
      R-WRCor & \textbf{2877} & 93.82 $\pm$ 0.38 & 93.26 $\pm$ 0.29$^\dagger$ \\
      RL-WRCor & 3188 & 93.85 $\pm$ 0.26$^\dagger$ & 93.36 $\pm$ 0.20$^\ddagger$ \\
      RE-WRCor & 3088 & 93.48 $\pm$ 0.08 & 93.10 $\pm$ 0.06 \\
      \hline
      R-ZeroCost & 4677 & 89.00 $\pm$ 2.87 & 88.46 $\pm$ 3.17 \\
      RL-ZeroCost & 5157 & 92.11 $\pm$ 1.00 & 91.37 $\pm$ 1.07 \\
      RE-ZeroCost & 4983 & 85.15 $\pm$ 0.14 & 84.29 $\pm$ 0.36 \\
      \hline
      R-SJW & 5831 & 93.61 $\pm$ 0.13 & 93.10 $\pm$ 0.14 \\
      RL-SJW & 6430 & \textbf{93.99 $\pm$ 0.13}$^\star$ & \textbf{93.47 $\pm$ 0.03}$^\star$ \\
      RE-SJW & 6204 & 93.89 $\pm$ 0.51$^\ddagger$ & 93.15 $\pm$ 0.77 \\
      \hline
  \end{tabular}
\end{table}

\subsubsection{Search on NAS-Bench-201 Search Space}

To further compare the existing zero-shot NAS algorithms with ours, we conduct experiments on the NAS-Bench-201 benchmark. Table~\ref{tab: NASBENCH201} reports their test accuracies and search costs on the CIFAR-10, CIFAR-100, and ImageNet16-120 datasets. X-WRCor outperforms most existing NAS algorithms, except GDAS, TE-NAS, and ZiCo on CIFAR-10 and CIFAR-100 datasets. RE-SJW significantly outperforms most algorithms, except TE-NAS \cite{TENAS} and ZiCo \cite{ZiCo} on the CIFAR-10 dataset, AZ-NAS \cite{AZNAS} on the ImageNet16-120 dataset, and RE-ValACC. RE-SJW exhibits superior performance on the CIFAR-100 and ImageNet16-120 datasets compared to the CIFAR-10 dataset. This can be because our proxies offer a more accurate reflection of real-world performance, especially on datasets with numerous categories. Another reason is that the datasets with fewer categories tend to have samples within a minibatch that are more similar, which can potentially disturb our proxies and affect their performance. In essence, our proposed proxies are more suitable for datasets with a large number of categories. Besides, RE-WRCor has a significant performance reduction compared to R-WRCor and RL-WRCor across the three datasets. This is probably because our algorithm with WRCor can easily obtain inferior local optima. SJW can leverage the strengths of various proxies to improve the solution quality. RE-SJW is a promising zero-shot NAS algorithm in terms of performance and efficiency, achieving competitive architectures with up to 50 times efficiency compared to X-ValAcc and existing NAS algorithms.
\begin{table}
  \caption{The comparison of zero-shot NAS across different search strategies and training-free proxies on the NAS-Bench-201 benchmark, including CIFAR-10, CIFAR-100, and ImageNet16-120 datasets. We run each algorithm three times where $N=1000$. $^\star$, $^\ddagger$, and $^\dagger$ represent the top-3 algorithms in the last three blocks. X-Y represents the zero-shot NAS algorithm that uses search strategy X and training-free proxy Y.}
  \label{tab: NASBENCH201}
  \centering
  \begin{tabular}{@{}l@{}rl@{}l@{}l@{}}
      \hline
      &&\multicolumn{3}{c}{Test Acc (\%)}\\
      \cmidrule(r){3-5}
      Algorithm & Cost (s) & c10 & c100 & im120 \\
      \hline
      ENAS \cite{ENAS} & 13315 & 54.30 $\pm$ 0.00 & 15.61 $\pm$ 0.00 & 16.32 $\pm$ 0.00 \\ 
      RSPS \cite{RSPS} & 7587 & 87.66 $\pm$ 1.69 & 58.33 $\pm$ 4.34 & 31.14 $\pm$ 3.88 \\ 
      DARTS-V1 \cite{DARTS} & 10890 & 54.30 $\pm$ 0.00 & 15.61 $\pm$ 0.00 & 16.32 $\pm$ 0.00 \\ 
      DARTS-V2 \cite{DARTS} & 29902 & 54.30 $\pm$ 0.00 & 15.61 $\pm$ 0.00 & 16.32 $\pm$ 0.00 \\ 
      GDAS \cite{GDAS} & 28926 & 93.51 $\pm$ 0.13 & 70.61 $\pm$ 0.26 & 41.84 $\pm$ 0.90 \\ 
      SETN \cite{SETN} & 31010 & 86.19 $\pm$ 4.63 & 56.87 $\pm$ 7.77 & 31.90 $\pm$ 4.07 \\ 
      \hline
      NASWOT \cite{NASWOT} & - & 91.78 $\pm$ 1.45 & 67.05 $\pm$ 2.89 & 37.07 $\pm$ 6.39 \\ 
      TE-NAS \cite{TENAS} & - & 93.90 $\pm$ 0.47 & 71.24 $\pm$ 0.56 & 42.38 $\pm$ 0.46 \\ 
      ZiCo \cite{ZiCo} & - & 94.00 $\pm$ 0.40 & 71.10 $\pm$ 0.30 & 41.80 $\pm$ 0.30 \\ 
      AZ-NAS \cite{AZNAS} & - & 93.53 $\pm$ 0.15 & 70.75 $\pm$ 0.48 & 45.43 $\pm$ 0.29 \\
      \hline
      R-ValAcc & 105992 & 94.04 $\pm$ 0.18 & 72.51 $\pm$ 0.58 & 45.59 $\pm$ 0.71 \\
      RL-ValAcc & 116656 & 93.85 $\pm$ 0.11 & 71.65 $\pm$ 0.55 & 44.95 $\pm$ 1.21 \\
      RE-ValAcc & 112827 & 94.37 $\pm$ 0.00 & 73.09 $\pm$ 0.00 & 46.33 $\pm$ 0.00 \\
      \hline
      Optimum & - & 94.37 & 73.51 & 47.31 \\
      \hline
      R-WRCor & \textbf{2412} & 91.95 $\pm$ 1.84 & 68.93 $\pm$ 2.30 & 41.80 $\pm$ 4.50 \\
      RL-WRCor & 2666 & 93.11 $\pm$ 0.34$^\ddagger$ & 70.34 $\pm$ 0.54$^\ddagger$ & \textbf{45.42 $\pm$ 0.56}$^\star$ \\
      RE-WRCor & 2569 & 89.37 $\pm$ 0.00 & 65.71 $\pm$ 0.00 & 35.48 $\pm$ 0.00 \\
      \hline
      R-ZeroCost & 5474 & 93.03 $\pm$ 0.44$^\dagger$ & 68.64 $\pm$ 2.09 & 41.31 $\pm$ 2.76 \\
      RL-ZeroCost & 6048 & 91.54 $\pm$ 2.03 & 66.35 $\pm$ 3.85 & 35.27 $\pm$ 6.68 \\
      RE-ZeroCost & 5842 & 92.23 $\pm$ 0.11 & 67.12 $\pm$ 0.26 & 34.48 $\pm$ 4.01 \\
      \hline
      R-SJW & 5825 & 92.04 $\pm$ 1.90 & 68.82 $\pm$ 2.40 & 41.31 $\pm$ 4.12 \\
      RL-SJW & 6572 & 92.96 $\pm$ 0.46 & 70.05 $\pm$ 0.72$^\dagger$ & 44.80 $\pm$ 1.12$^\dagger$ \\
      RE-SJW & 6334 & \textbf{93.76 $\pm$ 0.41}$^\star$ & \textbf{71.57 $\pm$ 1.27}$^\star$ & 45.22 $\pm$ 0.94$^\ddagger$ \\
      \hline
  \end{tabular}
\end{table}

\subsubsection{Search on MobileNetV2 Search Space}

We also conduct experiments of architecture search on the ImageNet-1k dataset in the MobileNetV2 search space. We employ the optimal RE-SJW according to the experimental results on the NAS-Bench-101/201 benchmarks. Table~\ref{tab: MOBILENETV2} reports the test error on the ImageNet-1k dataset, the number of parameters, and the number of FLOPs for our discovered architecture below 600M FLOPs, as well as the search cost. The discovered architecture by RE-SJW is detailed in Table~\ref{tab: DISCOVEREDARCH}. This architecture achieves a competitive result of 22.1\% top-1 test error on the ImageNet-1k dataset in four GPU hours, thus being more efficient than existing NAS algorithms. It outperforms all manually designed architectures and most architectures designed by existing NAS algorithms.
Compared with existing zero-shot NAS algorithms, our discovered architecture achieves 2.4\% less test error than TE-NAS \cite{TENAS} with the same search time and only 0.2\% and 0.5\% more test error than ZiCo \cite{ZiCo} and AZ-NAS \cite{AZNAS}, respectively, while requiring less search time.
We discover a comparable architecture with significantly less search cost. Therefore, RE-SJW is an effective and efficient zero-shot NAS algorithm.
\begin{table}
  \caption{The comparison to the discovered architectures below 600M FLOPs by different NAS algorithms on the ImageNet-1k dataset. Algo. Type refers to the type of the NAS algorithm, including manually designing (M), multi-shot NAS (MS), one-shot NAS (OS), zero-shot NAS (ZS).}
  \label{tab: MOBILENETV2}
  \centering
  \begin{tabular}{@{}l@{}c@{}c@{}cc@{}cc@{}}
    \hline
    &\multicolumn{2}{c}{Test Error (\%)}&Params&FLOPs&Cost&Algo. \\
    \cmidrule(r){2-3}
    Algorithm & Top-1 & Top-5 & (M) & (M) & (GPU Days) & Type \\
    \hline
    MobileNet-V2 \cite{MobileNetV2} & 25.3 & - & 6.9 & 585 & - & M \\
    ShuffleNet-V2 \cite{ShuffleNetV2} & 25.1 & - & 7.4 & 591 & - & M \\
    EfficientNet-B0 \cite{EfficientNet} & 23.7 & 6.8 & 5.3 & 390 & - & M \\
    \hline
    NASNet \cite{NASNet} & 26.0 & 8.4 & 5.3 & 564 & 1800 & MS \\
    AmoebaNet \cite{AmoebaNet} & 24.3 & 7.6 & 6.4 & 570 & 3150 & MS \\
    EcoNAS \cite{EcoNAS} & 25.2 & - & \textbf{4.3} & - & 8 & MS \\
    \hline
    PNAS \cite{PNAS} & 25.8 & 8.1 & 5.1 & 588 & $\sim$225 & FS \\
    SemiNAS \cite{SemiNAS} & 23.5 & 6.8 & 6.3 & 599 & 4 & FS \\
    \hline
    DARTS \cite{DARTS} & 26.7 & 8.7 & 4.7 & 574 & 4 & OS \\
    PDARTS \cite{PDARTS} & 24.1 & 7.3 & 5.4 & 597 & 2.0 & OS \\
    GNAS \cite{GNAS} & 25.7 & 8.1 & 5.3 & - & - & OS \\
    MR-DARTS \cite{MR-DARTS} & 24.0 & 7.0 & - & - & 0.5 & OS \\
    ScarletNAS \cite{ScarletNAS} & 23.1 & 6.6 & 6.7 & 365 & 10 & OS \\
    FBNetV2 \cite{FBNetV2} & 22.8 & - & - & \textbf{325} & 25 & OS \\
    OFA \cite{OFA} & \textbf{20.0} & - & - & 595 & 175 & OS \\
    DESEvo \cite{DESEvo} & 23.5 & 6.5 & 5.5 & - & 0.2 & OS \\
    \hline
    TE-NAS \cite{TENAS} & 24.5 & 7.5 & 5.4 & - & \textbf{0.17} & ZS \\
    ZenNAS \cite{ZenNAS} & 21.7 & - & 12.4 & 453 & 0.5 & ZS \\
    ZiCo \cite{ZiCo} & 21.9 & - & - & 448 & 0.4 & ZS \\
    AZ-NAS \cite{AZNAS} & 21.4 & - & - & 462 & 0.4 & ZS \\
    SED \cite{SED} & 25.9 & - & 6.0 & - & - & ZS \\
    AZ-NAS \cite{LSWAG} & 23.6 & - & 6.2 & - & 0.7 & ZS \\
    LIBRA \cite{LSWAG} & 23.1 & - & 5.7 & - & 0.3 & ZS \\
    \hline
    RE-SJW & 21.9 & \textbf{5.9} & 10.8 & 592 & \textbf{0.17} & ZS \\
    \hline
  \end{tabular}
\end{table}
\begin{table}
  \caption{The discovered architecture on the MobileNetV2 search space by RE-SJW.}
  \label{tab: DISCOVEREDARCH}
  \centering
  \begin{tabular}{ccccccccc}
      \hline
      Stage & Op & E & K & S & L & C & B & R \\
      \hline
      - & Conv & - & 3 & 2 & 1 & 48   & -   & 224 \\
      1 & MB   & 2 & 7 & 2 & 2 & 32   & 40  & 112 \\
      2 & MB   & 2 & 7 & 2 & 3 & 32   & 48  & 56 \\
      3 & MB   & 6 & 7 & 1 & 3 & 64   & 40  & 28 \\
      4 & MB   & 6 & 7 & 2 & 4 & 88   & 40  & 28 \\
      5 & MB   & 6 & 7 & 1 & 6 & 104  & 56  & 14 \\
      6 & MB   & 6 & 5 & 2 & 6 & 152  & 192 & 14 \\
      7 & MB   & 6 & 5 & 1 & 2 & 240  & 136 & 7 \\
      - & Conv & - & 1 & 1 & 1 & 1264 & -   & 7 \\
      \hline
  \end{tabular}
\end{table}

\section{Discussion}

Zero-shot NAS algorithms can discover superior architectures with less search cost, yet their performance diminishes when applied to larger search spaces. Therefore, devising a better zero-shot NAS algorithm is a promising direction for future research. Future works should focus on developing a better training-free proxy for zero-shot NAS.

Besides, following extensive literature reviews and experiments, we also have the same observation as the previous works \cite{ZeroCost, PPP}. There is no single proxy (including our proxies) that can consistently and robustly outperform other existing proxies in all situations. An optimal training-free proxy should be more comprehensive, considering various aspects of neural networks rather than just expressivity and generalizability. Furthermore, our experiments of voting proxies highlight the need for effective dynamic voting mechanisms, which is also an important direction for zero-shot NAS in the future.

The proposal of training-free proxies benefits from the development of interpretability and theories of deep learning. Meanwhile, zero-shot NAS can facilitate the exploration of more advanced architectures, thereby driving the continuous advancement of deep learning. Therefore, it is crucial to pay attention to the interpretability and theories of deep learning in the future.

Although the zero-shot NAS algorithm we proposed has only demonstrated its effectiveness in image recognition, it has the potential to be extended to other downstream tasks in computer vision, and even other architectures and fields.

\section{Conclusion}

We propose a set of training-free proxies, including ACor, RCor, and WRCor, for architecture estimation of zero-shot NAS, which can evaluate the layer-wise expressivity and generalizability of neural architectures by the correlation of responses. Experimental results on proxy evaluation demonstrate that our training-free proxies outperform most existing proxies in terms of efficiency, stability, and generality. Besides, we can achieve further improvement by using our voting proxies, i.e., SJW and SPW. Experimental results on architecture search show that our zero-shot NAS algorithms outperform most existing NAS algorithms for image recognition. Particularly, RE-SJW can achieve competitive results more efficiently on the ImageNet-1k dataset in the MobileNetV2 search space.

\section*{Acknowledgements}

This work was supported by the National Natural Science Foundation of China (Grant No. 62406003) and Science and Technology Innovation Program of Anhui Province (Grant No. 202423k09020020).



\bibliographystyle{elsarticle-num} 
\bibliography{refs}

\begin{thebibliography}{10}
\expandafter\ifx\csname url\endcsname\relax
  \def\url#1{\texttt{#1}}\fi
\expandafter\ifx\csname urlprefix\endcsname\relax\def\urlprefix{URL }\fi
\expandafter\ifx\csname href\endcsname\relax
  \def\href#1#2{#2} \def\path#1{#1}\fi

\bibitem{NASSurvey}
P.~Ren, Y.~Xiao, X.~Chang, P.~Huang, Z.~Li, X.~Chen, X.~Wang, A comprehensive survey of neural architecture search: Challenges and solutions, ACM Computing Surveys 54~(4) (2022) 76:1--76:34.

\bibitem{EvoAAE}
G.-Q. Zeng, Y.-W. Yang, K.-D. Lu, G.-G. Geng, J.~Weng, Evolutionary adversarial autoencoder for unsupervised anomaly detection of industrial internet of things, IEEE Transactions on Reliability (2025) 1--15.

\bibitem{MODEOCNN}
K.~Lu, J.~Huang, G.~Zeng, M.~Chen, G.~Geng, J.~Weng, Multi-objective discrete extremal optimization of variable-length blocks-based {CNN} by joint {NAS} and {HPO} for intrusion detection in iiot, IEEE Transactions on Dependable and Secure Computing 22~(4) (2025) 4266--4283.

\bibitem{AER}
K.~Jing, L.~Chen, J.~Xu, An architecture entropy regularizer for differentiable neural architecture search, Neural Networks 158 (2023) 111--120.

\bibitem{MR-DARTS}
F.~Gao, B.~Song, D.~Wang, H.~Qin, {MR-DARTS:} restricted connectivity differentiable architecture search in multi-path search space, Neurocomputing 482 (2022) 27--39.

\bibitem{SlimDARTS}
S.~Yin, B.~Niu, R.~Wang, X.~Wang, Spatial and channel level feature redundancy reduction for differentiable neural architecture search, Neurocomputing 630 (2025) 129713.

\bibitem{DESEvo}
J.~Zou, Y.~Liu, Y.~Liu, Y.~Xia, Evolutionary multi-objective neural architecture search via depth equalization supernet, Neurocomputing 633 (2025) 129674.

\bibitem{NAS}
B.~Zoph, Q.~V. Le, Neural architecture search with reinforcement learning, in: Proceedings of International Conference on Learning Representations, OpenReview.net, Toulon, France, 2017, pp. 1--16.

\bibitem{NASNet}
B.~Zoph, V.~Vasudevan, J.~Shlens, Q.~V. Le, Learning transferable architectures for scalable image recognition, in: Proceedings of IEEE Conference on Computer Vision and Pattern Recognition, {IEEE} Computer Society, Salt Lake City, UT, USA, 2018, pp. 8697--8710.

\bibitem{AmoebaNet}
E.~Real, A.~Aggarwal, Y.~Huang, Q.~V. Le, Regularized evolution for image classifier architecture search, in: Proceedings of AAAI Conference on Artificial Intelligence, {AAAI} Press, Honolulu, Hawaii, USA, 2019, pp. 4780--4789.

\bibitem{PNAS}
C.~Liu, B.~Zoph, M.~Neumann, J.~Shlens, W.~Hua, L.~Li, L.~Fei{-}Fei, A.~L. Yuille, J.~Huang, K.~Murphy, Progressive neural architecture search, in: Proceedings of European Conference on Computer Vision, Springer, Munich, Germany, 2018, pp. 19--35.

\bibitem{GMAENAS}
K.~Jing, J.~Xu, P.~Li, Graph masked autoencoder enhanced predictor for neural architecture search, in: Proceedings of International Joint Conference on Artificial Intelligence, ijcai.org, Vienna, Austria, 2022, pp. 3114--3120.

\bibitem{GATES}
X.~Ning, Y.~Zheng, T.~Zhao, Y.~Wang, H.~Yang, A generic graph-based neural architecture encoding scheme for predictor-based {NAS}, in: Proceedings of European Conference on Computer Vision, Springer, Glasgow, UK, 2020, pp. 189--204.

\bibitem{ENAS}
H.~Pham, M.~Y. Guan, B.~Zoph, Q.~V. Le, J.~Dean, Efficient neural architecture search via parameter sharing, in: Proceedings of International Conference on Machine Learning, {PMLR}, Stockholmsm{\"{a}}ssan, Stockholm, Sweden, 2018, pp. 4092--4101.

\bibitem{DARTS}
H.~Liu, K.~Simonyan, Y.~Yang, {DARTS:} differentiable architecture search, in: Proceedings of International Conference on Learning Representations, OpenReview.net, New Orleans, LA, USA, 2019, pp. 1--13.

\bibitem{NASWOT}
J.~Mellor, J.~Turner, A.~J. Storkey, E.~J. Crowley, Neural architecture search without training, in: Proceedings of International Conference on Machine Learning, {PMLR}, Virtual Event, 2021, pp. 7588--7598.

\bibitem{EPENAS}
V.~Lopes, S.~Alirezazadeh, L.~A. Alexandre, {EPE-NAS:} efficient performance estimation without training for neural architecture search, in: Proceedings of International Conference on Artificial Neural Networks, Springer, Bratislava, Slovakia, 2021, pp. 552--563.

\bibitem{ZeroShotNAS}
G.~Li, D.~Hoang, K.~Bhardwaj, M.~Lin, Z.~Wang, R.~Marculescu, Zero-shot neural architecture search: Challenges, solutions, and opportunities, arXiv preprint arXiv:2307.01998 (2023).

\bibitem{ZeroCost}
M.~S. Abdelfattah, A.~Mehrotra, L.~Dudziak, N.~D. Lane, Zero-cost proxies for lightweight {NAS}, in: Proceedings of International Conference on Learning Representations, OpenReview.net, Virtual Event, Austria, 2021, pp. 1--17.

\bibitem{ZenNAS}
M.~Lin, P.~Wang, Z.~Sun, H.~Chen, X.~Sun, Q.~Qian, H.~Li, R.~Jin, Zen-nas: {A} zero-shot {NAS} for high-performance image recognition, in: Proceedings of IEEE International Conference on Computer Vision, {IEEE} Computer Society, Montreal, QC, Canada, 2021, pp. 337--346.

\bibitem{ZiCo}
G.~Li, Y.~Yang, K.~Bhardwaj, R.~Marculescu, Zico: Zero-shot {NAS} via inverse coefficient of variation on gradients, in: Proceedings of International Conference on Learning Representations, OpenReview.net, Kigali, Rwanda, 2023, pp. 1--31.

\bibitem{DeeperZeroCost}
C.~White, M.~Khodak, R.~Tu, S.~Shah, S.~Bubeck, D.~Dey, A deeper look at zero-cost proxies for lightweight nas, in: ICLR Blog Track, 2022, pp. 1--1, https://iclr-blog-track.github.io/2022/03/25/zero-cost-proxies/.

\bibitem{TENAS}
W.~Chen, X.~Gong, Z.~Wang, Neural architecture search on imagenet in four {GPU} hours: {A} theoretically inspired perspective, in: Proceedings of International Conference on Learning Representations, OpenReview.net, Virtual Event, Austria, 2021, pp. 1--15.

\bibitem{SNIP}
N.~Lee, T.~Ajanthan, P.~Torr, Snip: Single-shot network pruning based on connection sensitivity, in: Proceedings of International Conference on Learning Representations, OpenReview.net, New Orleans, LA, USA, 2019, pp. 1--15.

\bibitem{GraSP}
C.~Wang, G.~Zhang, R.~B. Grosse, Picking winning tickets before training by preserving gradient flow, in: Proceedings of International Conference on Learning Representations, OpenReview.net, Addis Ababa, Ethiopia, 2020, pp. 1--11.

\bibitem{SynFlow}
H.~Tanaka, D.~Kunin, D.~L.~K. Yamins, S.~Ganguli, Pruning neural networks without any data by iteratively conserving synaptic flow, in: Proceedings of Conference on Neural Information Processing Systems, virtual, 2020, pp. 6377--6389.

\bibitem{Fisher}
L.~Theis, I.~Korshunova, A.~Tejani, F.~Husz{\'{a}}r, Faster gaze prediction with dense networks and fisher pruning, arXiv preprint arXiv:1801.05787 (2018).

\bibitem{LyrDynIsometry}
N.~Lee, T.~Ajanthan, S.~Gould, P.~H.~S. Torr, A signal propagation perspective for pruning neural networks at initialization, in: Proceedings of International Conference on Learning Representations, OpenReview.net, Addis Ababa, Ethiopia, 2020, pp. 1--16.

\bibitem{BlockSwap}
J.~Turner, E.~J. Crowley, M.~F.~P. O'Boyle, A.~J. Storkey, G.~Gray, Blockswap: Fisher-guided block substitution for network compression on a budget, in: Proceedings of International Conference on Learning Representations, OpenReview.net, Addis Ababa, Ethiopia, 2020, pp. 1--15.

\bibitem{BarlowTwins}
J.~Zbontar, L.~Jing, I.~Misra, Y.~LeCun, S.~Deny, Barlow twins: Self-supervised learning via redundancy reduction, in: Proceedings of International Conference on Machine Learning, {PMLR}, Virtual Event, 2021, pp. 12310--12320.

\bibitem{BundleEntropy}
D.~Peer, S.~Stabinger, A.~J. Rodr{\'{\i}}guez{-}S{\'{a}}nchez, Auto-tuning of deep neural networks by conflicting layer removal, arXiv preprint arXiv:2103.04331 (2021).

\bibitem{TFGB}
M.~Hardt, B.~Recht, Y.~Singer, Train faster, generalize better: Stability of stochastic gradient descent, in: Proceedings of International Conference on Machine Learning, Vol.~48 of {JMLR} Workshop and Conference Proceedings, 2016, pp. 1225--1234.

\bibitem{PPP}
C.~White, A.~Zela, R.~Ru, Y.~Liu, F.~Hutter, How powerful are performance predictors in neural architecture search?, in: Proceedings of Conference on Neural Information Processing Systems, virtual, 2021, pp. 28454--28469.

\bibitem{NASBench101}
C.~Ying, A.~Klein, E.~Christiansen, E.~Real, K.~Murphy, F.~Hutter, Nas-bench-101: Towards reproducible neural architecture search, in: Proceedings of International Conference on Machine Learning, {PMLR}, Long Beach, California, USA, 2019, pp. 7105--7114.

\bibitem{CIFAR}
A.~Krizhevsky, Learning multiple layers of features from tiny images, Master's thesis, University of Toronto (2009).

\bibitem{NASBench201}
X.~Dong, Y.~Yang, Nas-bench-201: Extending the scope of reproducible neural architecture search, in: Proceedings of International Conference on Learning Representations, OpenReview.net, Addis Ababa, Ethiopia, 2020, pp. 1--16.

\bibitem{DownsampledImageNet}
P.~Chrabaszcz, I.~Loshchilov, F.~Hutter, A downsampled variant of imagenet as an alternative to the {CIFAR} datasets, arXiv preprint arXiv:1707.08819 (2017).

\bibitem{MobileNetV2}
M.~Sandler, A.~G. Howard, M.~Zhu, A.~Zhmoginov, L.~Chen, Mobilenetv2: Inverted residuals and linear bottlenecks, in: Proceedings of IEEE Conference on Computer Vision and Pattern Recognition, {IEEE} Computer Society, Salt Lake City, UT, USA, 2018, pp. 4510--4520.

\bibitem{ImageNet}
J.~Deng, W.~Dong, R.~Socher, L.-J. Li, K.~Li, F.-F. Li, Imagenet: A large-scale hierarchical image database, in: Proceedings of IEEE Conference on Computer Vision and Pattern Recognition, {IEEE} Computer Society, Miami, Florida, USA, 2009, pp. 248--255.

\bibitem{KaimingInit}
K.~He, X.~Zhang, S.~Ren, J.~Sun, Delving deep into rectifiers: Surpassing human-level performance on imagenet classification, in: Proceedings of IEEE International Conference on Computer Vision, {IEEE} Computer Society, Santiago, Chile, 2015, pp. 1026--1034.

\bibitem{AZNAS}
J.~Lee, B.~Ham, {AZ-NAS:} assembling zero-cost proxies for network architecture search, in: Proceedings of IEEE Conference on Computer Vision and Pattern Recognition, 2024, pp. 5893--5903.

\bibitem{RSPS}
L.~Li, A.~Talwalkar, Random search and reproducibility for neural architecture search, in: Proceedings of Conference on Uncertainty in Artificial Intelligence, {AUAI} Press, Tel Aviv, Israel, 2019, pp. 367--377.

\bibitem{GDAS}
X.~Dong, Y.~Yang, Searching for a robust neural architecture in four {GPU} hours, in: Proceedings of IEEE Conference on Computer Vision and Pattern Recognition, {IEEE} Computer Society, Long Beach, CA, USA, 2019, pp. 1761--1770.

\bibitem{SETN}
X.~Dong, Y.~Yang, One-shot neural architecture search via self-evaluated template network, in: Proceedings of IEEE International Conference on Computer Vision, {IEEE} Computer Society, Seoul, Korea (South), 2019, pp. 3680--3689.

\bibitem{ShuffleNetV2}
N.~Ma, X.~Zhang, H.~Zheng, J.~Sun, Shufflenet {V2:} practical guidelines for efficient {CNN} architecture design, in: Proceedings of European Conference on Computer Vision, Springer, Munich, Germany, 2018, pp. 122--138.

\bibitem{EfficientNet}
M.~Tan, Q.~V. Le, Efficientnet: Rethinking model scaling for convolutional neural networks, in: Proceedings of International Conference on Machine Learning, {PMLR}, Long Beach, California, USA, 2019, pp. 6105--6114.

\bibitem{EcoNAS}
D.~Zhou, X.~Zhou, W.~Zhang, C.~C. Loy, S.~Yi, X.~Zhang, W.~Ouyang, Econas: Finding proxies for economical neural architecture search, in: Proceedings of IEEE Conference on Computer Vision and Pattern Recognition, {IEEE} Computer Society, Seattle, WA, USA, 2020, pp. 11393--11401.

\bibitem{SemiNAS}
R.~Luo, X.~Tan, R.~Wang, T.~Qin, E.~Chen, T.~Liu, Semi-supervised neural architecture search, in: Proceedings of Conference on Neural Information Processing Systems, virtual, 2020, pp. 10547--10557.

\bibitem{PDARTS}
X.~Chen, L.~Xie, J.~Wu, Q.~Tian, Progressive {DARTS:} bridging the optimization gap for {NAS} in the wild, International Journal of Computer Vision 129~(3) (2021) 638--655.

\bibitem{GNAS}
X.~Zhai, S.~Li, G.~Zhong, T.~Li, F.~Zhang, R.~Hedjam, Generative neural architecture search, Neurocomputing 642 (2025) 130360.

\bibitem{ScarletNAS}
X.~Chu, B.~Zhang, Q.~Li, R.~Xu, X.~Li, Scarlet-nas: Bridging the gap between scalability and fairness in neural architecture search, in: Proceedings of IEEE International Conference on Computer Vision Workshop, {IEEE} Computer Society, Montreal, QC, Canada, 2021, pp. 317--325.

\bibitem{FBNetV2}
A.~Wan, X.~Dai, P.~Zhang, Z.~He, Y.~Tian, S.~Xie, B.~Wu, M.~Yu, T.~Xu, K.~Chen, P.~Vajda, J.~E. Gonzalez, Fbnetv2: Differentiable neural architecture search for spatial and channel dimensions, in: Proceedings of IEEE Conference on Computer Vision and Pattern Recognition, {IEEE} Computer Society, Seattle, WA, USA, 2020, pp. 12962--12971.

\bibitem{OFA}
H.~Cai, C.~Gan, T.~Wang, Z.~Zhang, S.~Han, Once-for-all: Train one network and specialize it for efficient deployment, in: Proceedings of International Conference on Learning Representations, OpenReview.net, Addis Ababa, Ethiopia, 2020, pp. 1--15.

\bibitem{SED}
N.~Wu, H.~Huang, Y.~Xu, Z.~Hao, Zero-shot {NAS} via the suppression of local entropy decrease, arXiv preprint arXiv:2411.06236 (2024).

\bibitem{LSWAG}
S.~Casarin, S.~Escalera, O.~Lanz, {L-SWAG:} layer-sample wise activation with gradients information for zero-shot {NAS} on vision transformers, in: Proceedings of IEEE Conference on Computer Vision and Pattern Recognition, 2025, pp. 4441--4451.

\end{thebibliography}

\end{sloppypar}
\end{document}